\documentclass{article}

\usepackage{microtype}
\usepackage{graphicx}
\usepackage{subcaption}
\usepackage{booktabs} 

\usepackage{hyperref}

\usepackage[preprint]{icml2026}

\usepackage{amsmath}
\usepackage{amssymb}
\usepackage{mathtools}
\usepackage{amsthm}

\usepackage{enumitem}

\usepackage[capitalize,noabbrev]{cleveref}

\theoremstyle{plain}
\newtheorem{theorem}{Theorem}[section]
\newtheorem{proposition}[theorem]{Proposition}
\newtheorem{lemma}[theorem]{Lemma}

\theoremstyle{definition}
\newtheorem{definition}[theorem]{Definition}
\newtheorem{assumption}[theorem]{Assumption}
\theoremstyle{remark}
\newtheorem{remark}[theorem]{Remark}

\usepackage[textsize=tiny]{todonotes}

\newcommand{\doubleD}{\mathbb{D}}
\newcommand{\doubleE}{\mathbb{E}}

\newcommand{\doubleI}{\mathbb{I}}

\newcommand{\doubleP}{\mathbb{P}}

\newcommand{\doubleR}{\mathbb{R}}

\newcommand{\doubleV}{\mathbb{V}}

\newcommand{\doubleZ}{\mathbb{Z}}

\newcommand{\Acal}{\mathcal{A}}

\newcommand{\Mcal}{\mathcal{M}}

\newcommand{\Ocal}{\mathcal{O}}

\newcommand{\Rcal}{\mathcal{R}}
\newcommand{\Scal}{\mathcal{S}}

\newcommand{\Xcal}{\mathcal{X}}

\newcommand{\cbr}[1]{\left\{ #1 \right\}}

\newcommand{\intern}{$^\dagger$Work done partly during an internship at LY Corporation.}

\icmltitlerunning{A Relative-Budget Theory for Reinforcement Learning with Verifiable Rewards in LLM Reasoning}

\begin{document}

\twocolumn[
  \icmltitle{A Relative-Budget Theory for Reinforcement Learning with Verifiable Rewards in Large Language Model Reasoning}



  \icmlsetsymbol{equal}{*}
  \icmlsetsymbol{intern}{$\dagger$}

  \begin{icmlauthorlist}
    \icmlauthor{Akifumi Wachi}{ly}
    \icmlauthor{Hirota Kinoshita}{ttic,intern}
    \icmlauthor{Shokichi Takakura}{ly}
    \icmlauthor{Rei Higuchi}{ut,riken}
    \icmlauthor{Taiji Suzuki}{ut,riken}
  \end{icmlauthorlist}

  \icmlaffiliation{ly}{LY Corporation}
  \icmlaffiliation{ttic}{Toyota Technological Institute at Chicago}
  \icmlaffiliation{ut}{University of Tokyo}
  \icmlaffiliation{riken}{RIKEN AIP}

  \icmlcorrespondingauthor{Akifumi Wachi}{akifumi.wachi@lycorp.co.jp}

  \icmlkeywords{Machine Learning, ICML}

  \vskip 0.3in
]



\printAffiliationsAndNotice{\intern}  

\begin{abstract}
    Reinforcement learning (RL) is a dominant paradigm for improving the reasoning abilities of large language models, yet its effectiveness varies across tasks and compute budgets.
    We propose a \emph{relative-budget} theory explaining this variation through a single quantity called relative budget $\xi \coloneqq H/\mathbb{E}[T]$, where $H$ is the generation horizon (token budget) and $T$ denotes the number of tokens until the first correct solution under a base policy. 
    We show that $\xi$ determines sample efficiency by controlling reward variance and the likelihood of informative trajectories.
    Our analysis reveals three regimes:
    in the \emph{deficient} regime ($\xi \to 0$), informative trajectories are rare and the sample complexity explodes;
    in the \emph{balanced} regime ($\xi=\Theta(1)$), informative trajectories occur with non-negligible probability and RL is maximally sample-efficient;
    and in the \emph{ample} regime ($\xi \to \infty$), learning remains stable but marginal gains per iteration diminish.
    We further provide finite-sample guarantees for online RL that characterize learning progress across these regimes.
    Specifically, in a case study under idealized distributional assumptions, we show that the relative budget grows linearly over iterations.
    Our empirical results confirm these predictions in realistic settings, identifying a budget $\xi \in [1.5, 2.0]$ that maximizes learning efficiency and coincides with peak reasoning performance.
\end{abstract}

\section{Introduction}
\label{sec:intro}

Reinforcement learning with verifiable rewards (RLVR) has emerged as a powerful paradigm for improving the reasoning capabilities of large language models (LLMs).
When rewards can be automatically verified, RLVR supports long-horizon planning and iterative self-correction, such as OpenAI o-series~\cite{jaech2024openai}, DeepSeek-R1~\cite{guo2025deepseek}, and Kimi K1.5~\cite{team2025kimi}.
Notably, \citet{jaech2024openai} report that gains from continued RL training and increased test-time compute exhibit scaling behavior distinct from pre-training~\cite{kaplan2020scaling}.
\begin{figure}[t]
    \centering
    \includegraphics[width=0.94\linewidth]{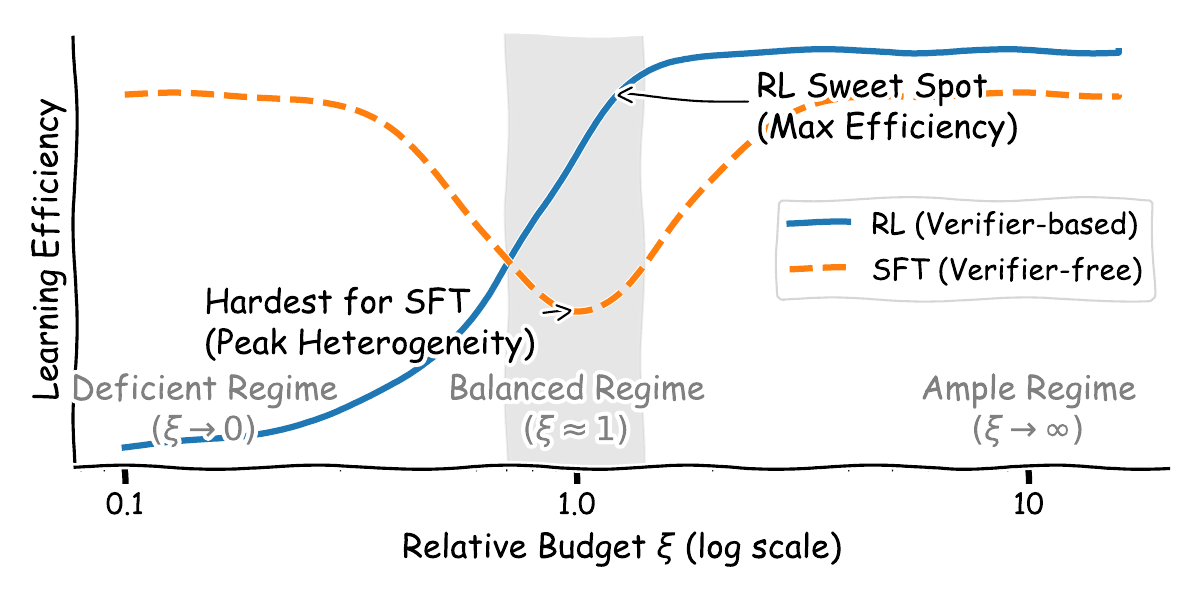}
    \caption{Conceptual illustration of our relative budget theory. The balanced relative budget regime ($\xi \approx 1$) acts as a phase transition point. In this regime, RL achieves maximal sample efficiency, whereas SFT suffers from performance degradation.}
    \label{fig:concept}
\end{figure}

Despite rapid empirical progress, theoretical understanding of RLVR under compute constraints remains under-explored.
Existing studies have primarily focused on the empirical development of algorithms~\cite{yu2025dapo} or data recipes~\citep{yephysics,zhou2025gsm}, with less emphasis on scaling guarantees.
A notable exception is \citet{setlurscaling}, who provided a foundational analysis showing that verifier-based approaches (e.g., RL) outperform verifier-free approaches (e.g., supervised fine-tuning, SFT) in scaling test-time compute. 
However, this framework does not account for the relationship between task difficulty and compute budget as well as the dynamics of online learning.
Without such guidelines, practitioners risk allocating vast computational resources to areas where learning is theoretically inefficient.

This relationship dictates the efficacy of RLVR.
Empirically, performance gains peak at the ``edge of competence''~\citep{zhang2025interplay}, where tasks are neither trivial nor intractable for the base model. 
Crucially, test-time compute scaling~\citep{snell2024scaling} renders difficulty dynamic as task tractability scales with the budget.
This shift necessitates a theoretical framework unifying capability and compute constraints to explain RL's regime-dependent success.

Our contribution is to analyze RLVR dynamics via the relative budget $\xi \coloneqq H/\doubleE[T]$, where $H$ is the token budget and $T$ denotes the number of tokens until the first correct solution.
Using $\xi$ to link the compute budgets and task difficulty, we present the following theoretical results:

\textbf{RL convergence in three regimes}~(\cref{thm:rl_suboptimality}). 
As shown in \cref{fig:concept}, we analyze the sub-optimality gap and sample complexity across three regimes governed by $\xi$:
\begin{itemize}[itemsep=2pt, parsep=0pt, topsep=1pt, partopsep=0pt]
\item \textbf{Deficient} ($\xi \to 0$): The probability of discovering high-reward trajectories vanishes, causing sample complexity to diverge.
\item \textbf{Balanced} ($\xi = \Theta(1)$): A critical divergence point where RL achieves maximal efficiency as the probability of high-reward trajectories is bounded away from $0$. Conversely, SFT suffers worst-case sub-optimality due to peak solution heterogeneity (\cref{sec:sft_comparison}).
\item \textbf{Ample} ($\xi \to \infty$): The probability of high-reward trajectories remains bounded away from zero, ensuring efficient learning though gains per iteration diminish.
\end{itemize}
Fundamentally, our core theory relies only on general assumptions, ensuring robustness to specific distributions.

\textbf{Finite-sample guarantees for online RL}~(\cref{thm:one-step-RL,thm:three-regimes_onlineRL,thm:online_rl}). 
Extending our analysis to the online setting, we derive the sample complexity required for monotonic policy improvement, which identifies the balanced regime as the optimal zone for sample efficiency. 
Furthermore, modeling reasoning time via a gamma distribution, we prove that the relative budget $\xi$ grows linearly over iterations, requiring increasingly many rollout samples.
This guarantees that the policy monotonically reduces its expected time-to-solution, effectively shifting the task towards the ample regime.

\textbf{Empirical validation.} We verify our theoretical findings through experiments with standard LLMs on reasoning benchmarks. 
We observe the predicted phase transition around $\xi \approx 1$. 
Moreover, the learning signal is maximized at $\xi \in [1.5, 2.0]$, which matches the optimal post-training performance. 
This confirms that our theory effectively explains optimal compute allocation in realistic tasks.

\section{Related Work}
\label{sec:related}

Recent advances in reasoning capabilities have shifted focus from simply scaling model parameters to optimizing the training data distribution and test-time compute allocation~\cite{alomrani2025reasoning}.
These empirical methods align with our theoretical analysis of the relative budget $\xi$.

\textbf{Test-time scaling and data efficiency.}
The trade-off between inference compute and performance has been formalized as test-time scaling laws~\citep{snell2024scaling,welleckdecoding}.
Practical implementations such as \textit{budget forcing}~\citep{muennighoff2025s1} manipulate the token budget $H$ by appending ``Wait'' tokens.
In our framework, this increases the relative budget $\xi$, effectively shifting problems from the \textit{deficient} regime ($\xi \ll 1$) where learning signals are absent to the \textit{balanced} regime ($\xi \approx 1$) where reasoning becomes feasible.
Alternatively, \citet{ye2025limo} and \citet{muennighoff2025s1} focus on the task distribution itself.
They show that reasoning capabilities emerge with minimal training examples (e.g., $\approx$1k samples) if trivial or unsolvable cases are removed.
These studies suggest that relative difficulty drives data efficiency more than raw scale.

\textbf{Difficulty filtering in RL.}
Aligning task difficulty with model capability is critical in RLVR~\cite{albalaksurvey,sun2025improving}. 
\citet{bae2025online} addressed this via online difficulty filtering, empirically demonstrating that training is most effective when samples are dynamically selected to have a success rate of $0.5$.
They categorized prompts into three regimes based on  difficulty levels, showing that maximizing the variance of the success rate leads to optimal learning efficiency.
This empirical result is closely related to our \cref{thm:rl_suboptimality,thm:three-regimes_onlineRL}, which identifies the balanced regime ($\xi \approx 1$) as the sweet spot for sample efficiency.
Similarly, \citet{cui2025process} proposed updating a process reward model online using outcome labels and then employed a filtering strategy to retain only prompts where the model's accuracy fell within $[0.2, 0.8]$.
Both \citet{bae2025online} and \citet{cui2025process} employ dynamic filtering to handle the ``moving target'' problem, where model improvements shift fixed tasks from the balanced to the ample regime. 

\section{Preliminaries and Problem Statement}
\label{sec:prelim}

\subsection{LLM Reasoning as Markov Decision Processes} 
Given a prompt (or problem) $x \in \Xcal$, an LLM generates an answer token by token.
Following~\citet{ke2025survey} and \citet{kazemnejad2025vineppo}, we model this process as a finite-horizon Markov decision process (MDP) defined as $\Mcal \coloneqq \langle\Scal, \Acal, F, r, H, \rho \rangle$,
where $\Scal$ is the state space, $\Acal$ is the token space, $F: \Scal \times \Acal \to \Scal$ is the deterministic state transition, $r: \Scal \times \Acal \to \cbr{0, 1}$ is a binary reward function in class $\Rcal$, 
$H \in \doubleZ_+$ is the horizon (i.e., token budget), and $\rho \in \Delta(\Xcal)$ is the initial prompt distribution, where $\Delta(\cdot)$ is the probability simplex.
For simplicity, we assume noiseless reward feedback.
Since decoding is autoregressive, each token corresponds to an action drawn from a policy $\pi: \Scal \to \Delta(\Acal)$.
Specifically, for all $h \in \{1, 2, \ldots, H\}$, let $s_h=(x,a_1,\dots,a_{h-1})$ and sample $a_h\sim\pi(\cdot\mid s_h)$. Thus
$\Scal \coloneqq \{ (x, a_1, \dots, a_{h-1}): x \in \Xcal, a_1, a_2, \ldots, a_{h-1} \in \Acal, h \in [H] \}$ with $s_1 =x$.
The state transition $F$ is deterministic; that is, upon producing token $a_h$, it simply appends the newly chosen action token to the current state by concatenation $s_{h+1} = (s_h, a_h)$ and collects reward $r_h \coloneqq r(s_h, a_h)$. 
A solution trace is a rollout  $\tau = (x, a_1, \dots, a_H)$ in the MDP, and the corresponding cumulative reward over $\tau$ is $R(\tau) = \sum_{h=1}^H r(s_h, a_h)$.

Following~\citet{setlurscaling}, we use a \emph{binary} reward with step shaping to encode a preference for early solutions.
Let $h^\star$ be the first step at which the trajectory reaches a correct solution. 
If $h^\star \le H$, we set $r_h=1$ for $h\ge h^\star$ and $r_h=0$ for $h<h^\star$, so the return is larger when the solution is found earlier, equivalent to imposing a length penalty~\cite{alomrani2025reasoning,arora2025training,liu2025learn}.
A natural alternative is to terminate immediately upon reaching a correct solution and assign a fixed terminal reward. 
This objective depends only on the success probability and does not distinguish early from late solutions, whereas our setting explicitly values earlier solutions under a token budget. 
We call a trace $\tau$ correct if it reaches the desired answer within $H$ tokens. 
In summary, we seek a policy $\pi$ that maximizes
\begin{align}
    \label{eq:optimization}
    J(\pi)
    \coloneqq \doubleE_{\rho, \pi} \left[ R(\tau)  \right] 
    = \mathbb{E}_{\rho, \pi} \left[ \sum_{h=1}^H r(s_h, a_h)  \right],
\end{align}
where $\doubleE_{\rho,\pi} [\cdot]$ abbreviates $\doubleE_{x \sim \rho} [\doubleE_{\tau \sim \pi(\cdot \mid x)} [\cdot]]$.
We use $\tau \sim \pi(\cdot \mid x)$ to denote a trajectory $\tau$ sampled from the distribution induced by $\pi$ starting at $x$.

\textbf{Notations.}
Throughout the paper, we use the standard asymptotic notation: $\Ocal(\cdot)$ for upper bounds, $\Omega(\cdot)$ for lower bounds, and $\Theta(\cdot)$ when both hold.
We write $a = \widetilde{\Ocal}(b)$ and $a = \widetilde{\Omega}(b)$ to respectively indicate $a = \Ocal(b \cdot \max(1, \mathrm{polylog}(b)))$ and $a = \Omega(b \cdot \max(1, \mathrm{polylog}(b)))$.
Additionally, we let $a \lesssim b$ denote $a = \Ocal(b)$. 
For a positive integer $n \in \doubleZ_+$, $[n]$ denotes the set $\{1, 2, . . . , n\}$.
Finally, $[\mathbf{x}]_{\mathbf{y}}^{\mathbf{z}} \coloneqq \min\{\mathbf{z},\max\{\mathbf{x}, \mathbf{y}\}\}$ denotes clipping.

\subsection{Sub-optimality Gap in RL}

\citet{setlurscaling} demonstrate that verifier-based RL methods theoretically outperform verifier-free baselines (e.g., SFT) in scaling test-time compute. While SFT is limited by solution heterogeneity, RL can exploit rare high-reward traces sampled by a base policy $\pi_b: \Scal \to \Delta(\Acal)$.
Motivated by this theoretical superiority, we focus on the RL setting.\footnote{We also discuss connections to SFT in \cref{sec:sft_comparison}.}

Following \citet{setlurscaling}, we compare against the best policy inside a $\chi^2$ trust region around $\pi_b$.
For any instance $x\in\Xcal$, define the instance-wise $\chi^2$-divergence
$\doubleD_{\chi^{2}}\bigl(\pi(\cdot \mid x)\,\|\,\pi_b(\cdot \mid x)\bigr)
\coloneqq
\mathbb{E}_{\tau\sim\pi_b (\cdot \mid x)}\left[
\left(\frac{\pi(\tau\mid x)}{\pi_b(\tau\mid x)} - 1\right)^2
\right]$.
Given a radius $\kappa>0$ and an instance distribution $\rho$ over $\Xcal$, define
$\Pi_{\kappa} \coloneqq \left\{\pi \,\middle|\,
\mathbb{E}_{x\sim\rho} \left[
\mathbb{D}_{\chi^{2}} \bigl(\pi(\cdot\mid x)\,\|\,\pi_b(\cdot\mid x)\bigr)
\right] \le \kappa
\right\}$ and 
$\bar\pi_{\kappa}\in\arg\max_{\pi\in\Pi_\kappa} J(\pi)$.
Although $\Pi_\kappa$ implies an average constraint over $x\sim\rho$, the comparator induces an instance-wise divergence level
\begin{equation}
\label{eq:kappa_x_def}
\kappa_x
\coloneqq
\mathbb{D}_{\chi^{2}}\!\bigl(\bar{\pi}_{\kappa}(\cdot \mid x)\,\|\,\pi_b(\cdot \mid x)\bigr),
\qquad \forall x\in\Xcal.
\end{equation}

Define $\sigma_{b,x}^2 \coloneqq \doubleV_{\tau \sim \pi_b(\cdot \mid x)}[R(\tau)]$.
Then, on instance $x$, a $\chi^2$ radius $\kappa_x$ limits the expected improvement to $\Ocal(\sigma_{b,x}\sqrt{\kappa_x})$.
This motivates measuring the probability of events that are $\sqrt{\kappa_x}$ standard deviations away from the mean.

\begin{definition}[Anti-concentration]
\label{def:anticonc}
Let $\pi_{b}$ be the base policy.
For each $x \in \Xcal$ and scalar $\varepsilon>0$, define
\begin{equation*}
    c_x(\varepsilon)
    \coloneqq \doubleP_{\tau\sim\pi_b(\cdot\mid x)}
    \Bigl[
       R(\tau)\ge
       \mathbb{E}_{\tau\sim\pi_b(\cdot \mid x)}[R(\tau)]
       +\sigma_{b,x}\sqrt{\varepsilon}
     \Bigr].
\end{equation*}
Then, with a $\chi^2$ radius $\kappa > 0$,
the \emph{anti-concentration coefficient} of $\pi_{b}$ is defined as $c_0(\kappa) \coloneqq \min_{x \in \Xcal} c_x(\kappa_x)$.
\end{definition}
The significance of anti-concentration is formalized by the following bound on the sub-optimality gap.
\begin{lemma}[Sub-optimality gap of RL, Theorem 5.7 in \citet{setlurscaling}]
\label{thm:upper}
With probability at least \(1-\delta\), the policy
\(\hat\pi^{\mathrm{RL}}_{n}\) returned by the RL algorithm proposed in \citet{setlurscaling} achieves the following sub-optimality gap with respect to the best comparator $\bar\pi_{\kappa}$:
\begin{align}
  J(\bar\pi_{\kappa})-
  J(\hat{\pi}_{n}^\mathrm{RL})
  \lesssim
  \frac{H\log(|\mathcal R|/\delta)}{c_{0}(\kappa)\,n}.
\end{align}
Here, $c_0(\kappa)$ is the anti-concentration coefficient of the base policy $\pi_{b}$ with a trust region radius $\kappa>0$ such that 
\begin{align} 
\label{eq:kappa_condition}
& \sqrt{\kappa_x} \cdot \mathbb{E}_{\tau\sim\pi_b (\cdot \mid x)}[R(\tau)] \le \sigma_{b,x} 
& \forall x \in \mathcal{X}.
\end{align}
\end{lemma}
Intuitively, a larger anti-concentration coefficient leads to a higher probability of discovering reward-rich solution traces during sampling, thereby theoretically reducing the sample complexity required for the RL agent.

\section{The Relative Budget Framework}
\label{sec:model-time-to-solution}

\subsection{Modeling Time to Correct Solutions}
We model the number of tokens until the first correct solution under the base policy $\pi_b$, and then relate its distributional properties to the reward variance and sample complexity of verifier-based RL.
Given a verifier $W_x: \Scal \to \{0, 1\}$ and solution trace $\tau$, we define $T(\tau) \coloneqq \min\{h \in \doubleZ_+ : W_x(s_h) = 1\}$, with $T(\tau) = \infty$ if unsolved.
Intuitively, more difficult problems for $\pi_b$ correspond to a larger expected number of tokens until the first correct solution.
We denote the first two moments of $T(\tau)$ as
\begin{align*}
    \mu_x \coloneqq \doubleE_{\tau \sim \pi_b(\cdot \mid x)}[T(\tau)] 
    \quad \text{and} \quad
    v_x \coloneqq \doubleV_{\tau \sim \pi_b(\cdot \mid x)}[T(\tau)].
\end{align*}
and the success probability for a problem $x$ under $\pi_b$ as
\begin{align}
    q_x \coloneqq \doubleP_{\tau \sim \pi_b(\cdot \mid x)}\bigl[T(\tau)\le H\bigr].
\end{align}
Finally, based on the bi-level reward structure, the cumulative reward $R(\tau)$ is determined solely by $T(\tau)$ as follows:
\begin{equation}
\label{eq:R_tau_x}
R(\tau) = 
\begin{cases} 
H - T(\tau) + 1 & \text{\quad if \quad} T(\tau) \le H, \\
0 & \text{\quad otherwise.}
\end{cases}
\end{equation}

\subsection{Defining Relative Budget and Assumptions}
To analyze the interplay between compute budget and problem difficulty, we introduce the \textit{relative budget}.
This parameter indicates how ample the budget $H$ is relative to the difficulty of a specific problem $x$ for the base policy $\pi_b$.
\begin{definition}[Relative budget]
    \label{definition:criticality_x}
    For a given problem $x$, we normalize the compute budget $H \in \doubleZ_+$ by its mean time-to-solution $\mu_x$ and define the relative budget (i.e., compute budget relative to task difficulty) as
    \begin{equation}
        \xi_x \coloneqq \frac{H}{\mu_x} = \frac{H}{\doubleE_{\tau \sim \pi_b(\cdot \mid x)}[T(\tau)]}.
    \end{equation}
\end{definition}

To ensure mathematical tractability, we restrict our analysis to the subset of solvable problems $\mathcal{X}_\sharp \subseteq \mathcal{X}$, where $\pi_b$ exhibits a finite expected time-to-solution (i.e., $\mu_x < \infty$ for all $x \in \mathcal{X}_\sharp$).
This restriction is empirically grounded, as recent studies~\cite{kim2025reinforcement,yue2025does,wu2025invisible} demonstrate that RLVR primarily optimizes efficiency on solvable tasks rather than providing new capabilities for fundamentally unsolvable ones.
Thus, we treat inherently unsolvable instances as outside our scope. 
We assume finite moments for $T(\tau)$ on $\mathcal{X}_\sharp$ as follows.

\begin{assumption}[Moments of $T(\tau)$]
We assume that $v_x = \Theta(\mu_x^2)$ and that the second moment
$\doubleE_{\tau \sim \pi_b(\cdot \mid x)}[T(\tau)^2]$ is finite for all $x \in \Xcal_\sharp$.
\label{ass:moments}
\end{assumption}
Intuitively, this assumption rules out heavy-tailed behaviors, ensuring the standard deviation of the time-to-solution scales with its mean.
A representative sufficient condition is a uniform sub-Weibull tail~\cite{vladimirova2020sub}.
By focusing on $\mathcal{X}_\sharp$, we address the efficiency of budget allocation rather than task feasibility.

To analyze the probability of discovering high-reward solution traces, we characterize the lower tail of $T(\tau)$.
\begin{assumption}[Left-tail of $T(\tau)$]
\label{assumption:left_tail_T_tau}
There exist constants $z_0\in(0,1]$, $c_-,c_+>0$, and a nondecreasing function
$f:[0,z_0]\to[0,1]$ with $f(0)=0$ and $f(z_1)>0$ for some $z_1\in(0,z_0]$ such that for every problem $x\in\mathcal{X}_\sharp$ and
every $t\le\mu_x$ with $z_x \coloneqq t/\mu_x \in [0,z_0]$,
\begin{align}
  c_-\,f(z_x)\ \le\ \doubleP_{\tau \sim \pi_b(\cdot \mid x)}[T(\tau)\le t]\ \le\ c_+\,f(z_x).
\end{align}
Also, $f$ satisfies a mild doubling condition near the origin.
Specifically, there exists $D \ge 1$ such that for all $z\in(0,z_0]$, 
$D^{-1} f(z) \le f(z/2) \le D f(z)$.
\end{assumption}
Here, $f(z)$ governs the rarity of solving a problem within a fraction $z$ of the expected time.
The doubling condition ensures $f$ decays polynomially rather than exponentially near the origin, ruling out distributions where early solutions are effectively impossible.

We next characterize the \textit{balanced} regime, where the compute budget is comparable to the expected time-to-solution.
\begin{assumption}[Non-degeneracy with balanced $\xi$]
\label{ass:balanced-nondeg-min}
Fix constants $0<\xi_{\min}<\xi_{\max}<\infty$ and define the set of instances with $\xi_x \in [\xi_{\min}, \xi_{\max}]$
\begin{align}
    \Xcal_b \coloneqq \Bigl\{x \in \Xcal: \xi_x = H/\mu_x \in[\xi_{\min}, \xi_{\max}]\Bigr\}.
\end{align}
For each $x \in \Xcal_b$, we assume there exist universal constants $q_{\min}\in(0, 1/2)$ and $c_v \in (0, 1]$ such that
$q_{\min} \le q_x \le 1-q_{\min}$
and the conditional variance satisfies
\begin{align}
    \mathbb{V}_{\tau \sim \pi_b(\cdot \mid x)}[T(\tau) \mid T(\tau)\le H] \ge c_v v_x.
    \label{eq:balanced-variance}
\end{align}
\end{assumption}
This assumption posits that in the balanced regime, the truncation at $H$ does not collapse the variance of the reasoning time distribution; that is, the difficulty distribution remains non-degenerate even conditioned on success.

\section{Sub-optimality Gap of RL on Different Relative Budget Regimes}
\label{sec:rl}

\subsection{Regime-Dependent Anti-Concentration Bounds}
Since RL learns from both successful and failed solution traces,
the expectation and variance of the cumulative reward $R(\tau)$ must be taken for the entire trajectory distribution, including failure cases $T(\tau) > H$ with $R(\tau) = 0$.

We now analyze the behavior of the anti-concentration coefficient
$c_0$ as a function of the relative budget $\xi_x = H / \mu_x$
by modeling the time to correct solution as discussed in \cref{sec:model-time-to-solution}. 
By analyzing the reward statistics via the success probability $q_x$, we relate $c_x(\varepsilon)$ to the left tail of $T(\tau)$.

\begin{lemma}[Anti-concentration in terms of $T(\tau)$]
\label{lemma:anticonc_T_tail}
For each $\varepsilon>0$ and instance $x$ such that $q_x > 0$, there exists a scalar coefficient $\eta(x,\varepsilon)$ such that
\begin{equation*}
    c_x(\varepsilon)
    =
    \doubleP_{\tau \sim \pi_b(\cdot \mid x)} \left[T(\tau) \le \eta(x,\varepsilon) \cdot \doubleE_{\tau \sim \pi_b(\cdot \mid x)}[T(\tau)]\right].
\end{equation*}
Also, for some constant $c_->0$ independent of $x$, we have
$c_x(\varepsilon) \ge c_-\,f\bigl(\bar{\eta}(x,\varepsilon)\bigr)$,
where $\bar{\eta}(x, \varepsilon) \coloneqq [\eta(x, \varepsilon)]_0^{z_0}$.
Note that $z_0$ is the constant introduced in \cref{assumption:left_tail_T_tau}.
\end{lemma}

We defer the full proof to \cref{proof_lemma:anticonc_T_tail}.
In our bi-level setting, higher rewards correspond to earlier solutions.
Consequently, the event where the reward exceeds its mean by $\sigma_{b,x} \sqrt{\varepsilon}$ is equivalent to solving the instance within a fraction $\eta(x,\varepsilon)$ of its mean time $\mu_x$.
\cref{assumption:left_tail_T_tau} then lower-bounds this probability by $c_-\,f(\bar\eta(x,\varepsilon))$.

We next analyze the scaling of $\sigma_{b,x}$ and $c_0$
across three regimes of the relative budget $\xi_x$.
For $\xi > 0$, we define the following subset of $\Xcal$ written as:
\begin{align}
  \mathcal{X}(\xi)
  \coloneqq \bigl\{x \,\bigm|\,
      \xi/2 \le \xi_x \le 2\xi
    \bigr\},
\end{align}
and the corresponding anti-concentration coefficient
\begin{align}
  c_0(\xi; \varepsilon)
  \coloneqq \inf_{x \in \mathcal{X}(\xi)} c_x(\varepsilon).
\end{align}
Accordingly, we define the objective conditioned on the relative budget slice $\mathcal{X}(\xi)$ as:
\begin{align}
    J_{\xi}(\pi) \coloneqq \doubleE_{x \sim \rho(\cdot | x \in \Xcal(\xi)), \tau \sim \pi(\cdot | x)}[R(\tau)].
\end{align}

\cref{thm:upper} establishes that the sample efficiency of RL is governed by the anti-concentration coefficient $c_0$.
To understand the impact of compute constraints, we now analyze the dependence of $c_0$ on the relative budget $\xi$.
The following theorem characterizes the asymptotic behavior of $c_0(\xi)$ across three distinct regimes, thereby determining the scaling of the sub-optimality gap.

\begin{figure*}[t]
  \centering
  \begin{minipage}[t]{0.65\textwidth}
  \vspace{0pt}
  \begin{subfigure}[t]{0.48\textwidth}
    \centering
    \includegraphics[height=3.9cm]{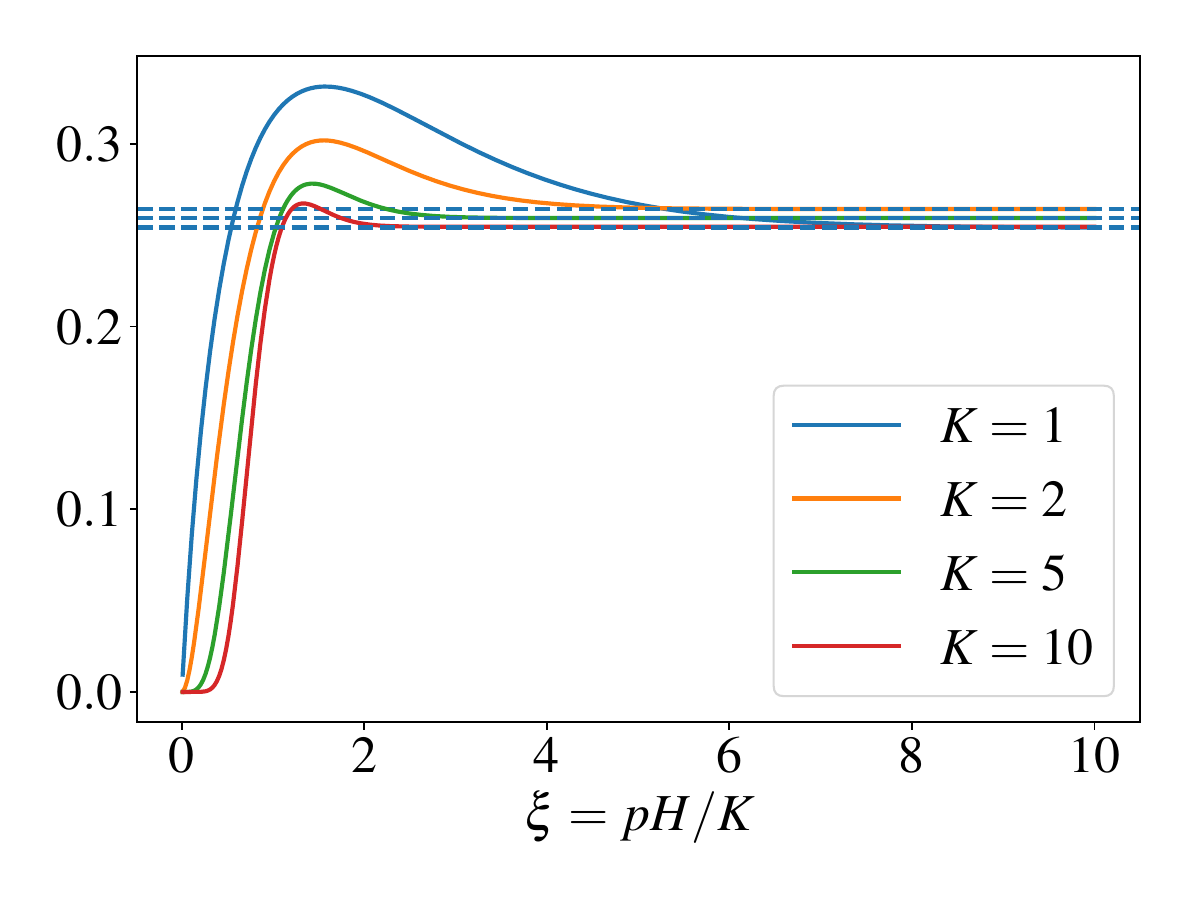}
    \caption{$\varepsilon=0.5$.}
    \label{fig:c0_eps_rl_1}
  \end{subfigure}
  \begin{subfigure}[t]{0.48\textwidth}
    \centering
    \includegraphics[height=3.9cm]{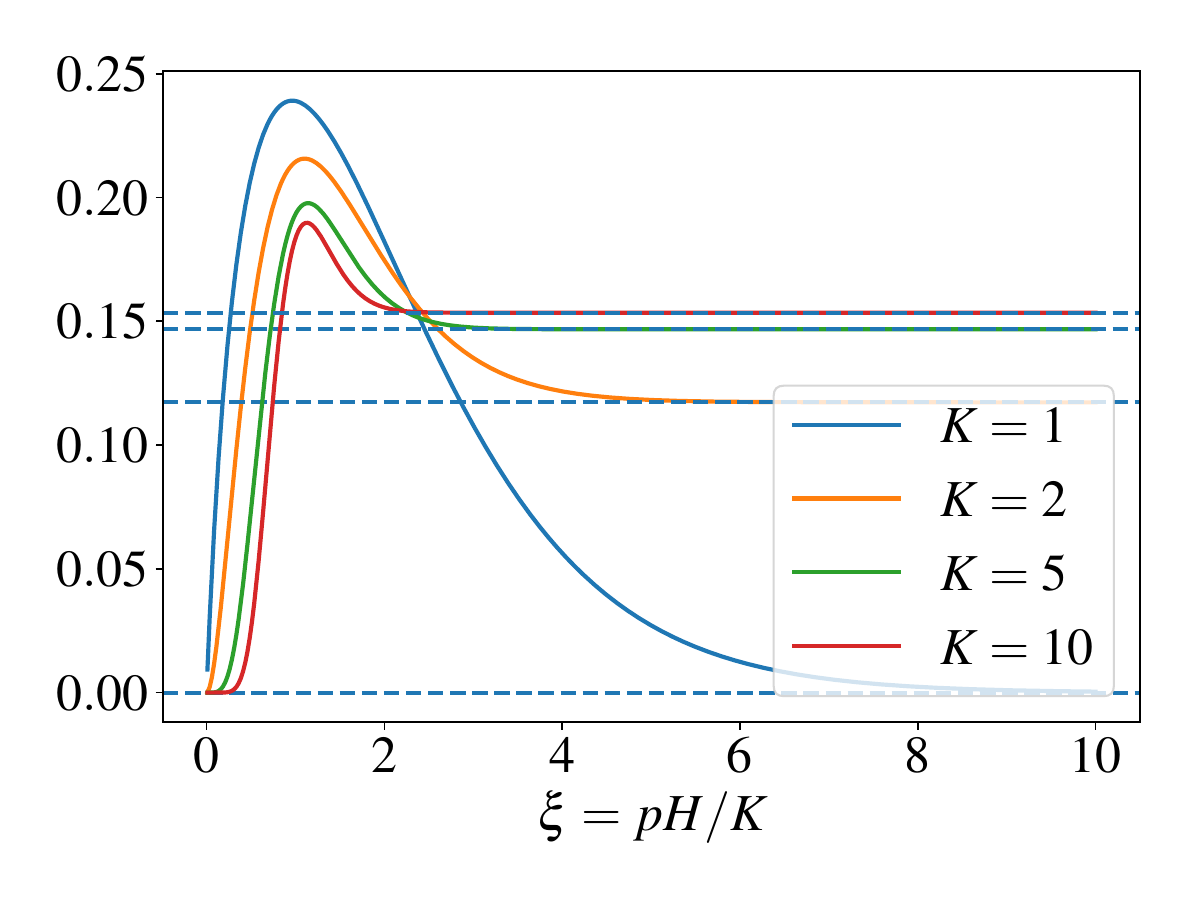}
    \caption{$\varepsilon=1.0$.}
    \label{fig:c0_eps_rl_2}
  \end{subfigure}
  \caption{Relations between $\xi$ and $c_0(K, \xi; \varepsilon)$ with different $K$ and $\varepsilon$.}
    \label{fig:c0_rl}
  \end{minipage}
  \begin{minipage}[t]{0.32\textwidth}
  \vspace{7pt}
    \centering
    \includegraphics[height=3.65cm]{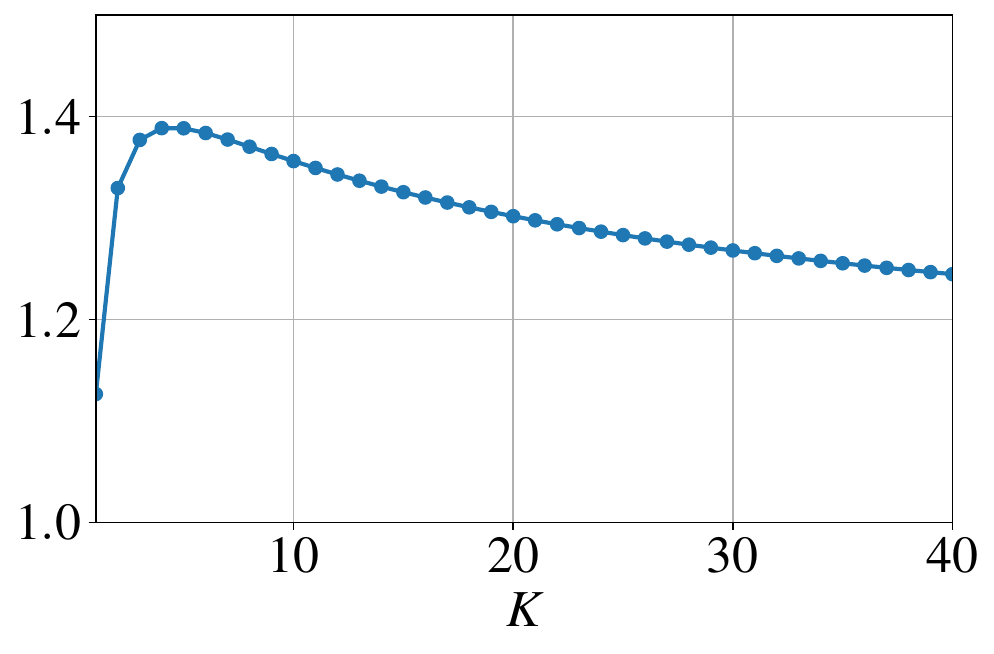}
    \vspace{-6pt}
  \caption{Optimal  relative budget $\xi^\star(K)$ depending on $K$.}
  \label{fig:optimal_xi}
  \end{minipage}
\end{figure*}

\begin{theorem}[Regime-dependent anti-concentration and an implication for RL]
\label{thm:rl_suboptimality}
Fix a relative-budget slice $\Xcal(\xi) \coloneqq \{x:\ \xi/2\le \xi_x \le 2\xi\}$.
Let $\hat\pi^{\mathrm{RL}}_n$ be a policy returned by the RL procedure  in \cref{thm:upper}.
With a $\chi^2$ trust-region radius $\kappa>0$, define $c_0(\xi; \kappa) \coloneqq \inf_{x\in \Xcal(\xi)} c_x(\kappa)$.
Provided that the trust region condition \eqref{eq:kappa_condition} is satisfied and that the comparator $\bar{\pi}_\kappa$ satisfies $\kappa_x \le \kappa$ for all $x \in \Xcal(\xi)$, the following bound holds with probability at least $1-\delta$:
\begin{align}
    J_{\xi}(\bar\pi_\kappa) - J_{\xi}(\hat\pi^{\mathrm{RL}}_n)
    \lesssim\ \frac{H\log(|\mathcal R|/\delta)}{c_0(\xi;\kappa)\,n}.
    \label{eq:suboptimality_c0}
\end{align}
Under Assumptions~\ref{ass:moments} and \ref{assumption:left_tail_T_tau}, the regime-dependent behavior of $c_0(\xi;\kappa)$ is:
\begin{enumerate}[itemsep=2pt, parsep=0pt, topsep=2pt, partopsep=0pt]
    \item (Deficient) $c_0(\xi;\kappa)=\Theta(f(\xi))$ as $\xi\to 0$.
    \item (Balanced) $c_0(\xi;\kappa)=\Theta(1)$ for $\xi\in[\xi_{\min},\xi_{\max}]$ if \cref{ass:balanced-nondeg-min} additionally holds and $\kappa$ is chosen so that $\kappa \le (q_{\min}/2)^2$.
    \item (Ample) $c_0(\xi;\kappa)=\Theta(1)$ as $\xi\to\infty$ for all $\kappa\le \bar{\kappa}$ for some constant $\bar{\kappa}\in(0,1)$.
\end{enumerate}
\end{theorem}
The proof is in \cref{appendix:proof_thm:rl_suboptimality}.
This theorem links RL performance to the relative budget $\xi$, offering a unified view of reasoning efficiency.
Crucially, in the balanced-to-ample regimes, where the compute budget $H$ is comparable to or exceeds the expected time-to-solution $\mu_x$, the anti-concentration coefficient remains strictly bounded away from $0$; that is, $c_0(\xi; \kappa) = \Theta(1)$.
This ensures that the sample complexity remains tractable, with error scaling linearly with the horizon $H$.
In contrast, in the deficient relative budget regime, learning efficiency degrades according to the solution's tail probability $f(\xi)$, quantifying the difficulty of learning from rare successful traces.

Our theoretical analysis also identifies a fundamental limitation: applying RL in the deficient regime ($\xi \ll 1$) is computationally prohibitive due to vanishing anti-concentration.
From the perspective of our framework, standard practices such as preliminary SFT or distillation can be interpreted as mechanisms that shift tasks into the balanced-to-ample regimes ($\xi \gtrsim 1$) by reducing the effective task difficulty $\mu_x$, thereby complementing the increase of compute budgets $H$.

\subsection{Concrete Analysis with Gamma Distribution}
\label{sec:gamma_statistics}

To provide concrete estimates, we employ the gamma distribution to model $T(\tau)$.
This choice satisfies our assumptions and offers mathematical tractability.
While reasoning steps are discrete, a continuous relaxation is justified by the large token counts in complex tasks; indeed, the discrete Negative Binomial distribution converges to the gamma distribution in the limit of low per-step success probability.

We characterize problems by the number of required insights $K \in \doubleZ_+$ and a success probability per step $p \in [0, 1]$.
Specifically, we assume $T(\tau) \mid x \sim \mathsf{Gamma}(K, p)$ with shape $K$ and rate $p$, yielding an expected time $\doubleE[T(\tau)] = K/p$.
For analytical convenience, we use the proxy reward $R(\tau) \coloneqq \max\{0, H-T(\tau)\}$, which differs from the original objective by at most an additive constant on successful trajectories.
This model generalizes the theoretical framework of \citet{kim2025metastable} and aligns with our empirical observations in \Cref{appendix:token_dist}.
Under this model, we quantitatively characterize the reward variance and anti-concentration.

\begin{proposition}[Statistics under Gamma model]
\label{prop:gamma_statistics}
Let $\gamma(\cdot, \cdot)$ denote the regularized lower incomplete gamma function.
Although the exact expressions involve incomplete gamma functions and are deferred to \cref{proof_lemma:sigma_RL_gamma,proof_lemma:anti-concentration-gamma}, we identify closed-form coefficients $C_\mathrm{RL}(K,\xi)$ and $\psi(K,\xi,\varepsilon)$ that characterize the scaling.
Specifically, the standard deviation $\sigma_b$ and anti-concentration coefficient $c$ satisfy:
\begin{align}
    \sigma_b(K,\xi) &= C_{\mathrm{RL}}(K,\xi)H \cdot (K\xi)^{-1}, \\
    c(K,\xi,\varepsilon) &= \gamma \bigl(K, K\xi \cdot [\psi(K,\xi,\varepsilon)]_0^1 \bigr).
\end{align}
\end{proposition}

The proof is provided in \cref{proof_lemma:sigma_RL_gamma,proof_lemma:anti-concentration-gamma}.
Notably, these explicit forms allow us to analyze the asymptotic behavior and optimal budget allocation.
\begin{remark}[Asymptotics]
\label{remark:asymptotics}
As $\xi\to\infty$, $C_{\mathrm{RL}}(K,\xi) \to \sqrt{K}$ holds and the anti-concentration coefficient converges to
$\lim_{\xi\to\infty} c(K,\xi,\varepsilon) = \gamma (K,\,K-\sqrt{\varepsilon K})$.
For a fixed $\varepsilon$, the coefficient does \emph{not} vanish with large $\xi$ unless $\varepsilon\ge K$, ensuring stable learning in the ample budget regime.
\end{remark}

\Cref{fig:c0_rl} illustrates the relation between the relative budget $\xi = pH/K$ and the coefficient $c_0$.
Consistent with \cref{remark:asymptotics}, $c_0$ remains strictly positive in the regime of $\xi = \Omega(1)$.
Furthermore, numerical analysis of $\sigma_b(K, \xi)$ reveals a unique global maximizer $\xi^\star(K)$.
As shown in \Cref{fig:optimal_xi}, the optimal relative budget falls in the range $1.0 \le \xi^\star(K) \le 1.4$ for various $K$.
This implies that the sub-optimality gap is minimized when the compute budget $H$ is set slightly larger than the expected time to correct solution. 
For more details, see~\cref{optimal_opt_xi}.

\subsection{Theoretical Comparison: RL vs. SFT}
\label{sec:sft_comparison}

Contrasting our findings with verifier-free SFT reveals a critical trade-off.
As established by \citet{setlurscaling}, the sub-optimality gap of SFT is lower-bounded by the \textit{policy heterogeneity} conditioned on success, denoted as $\sigma_b^{\mathrm{SFT}}(\xi)$.
Formally, its sub-optimality gap scales as $\Omega( \sigma_b^{\mathrm{SFT}}(\xi)/\sqrt{\bar{n}})$, where $\bar{n}$ is the number of successful solution traces.
This implies that learning is impeded when the successful demonstrations are diverse and lack a coherent strategy.

Incorporating the relative budget framework reveals a sharp contrast in the balanced regime ($\xi \approx 1$).
While RL achieves peak sample efficiency due to stable anti-concentration, SFT faces its worst-case scenario.
Our analysis in \cref{appendix:sft} shows that $\sigma_b^{\mathrm{SFT}}(\xi)$ is maximized at $\xi = \Theta(1)$.
Unlike the deficient or ample regimes where solution variance is low, the balanced regime yields a maximally heterogeneous mix of ``lucky'' short solutions and those barely fitting the budget.
Thus, the ``sweet spot'' for verifier-based RL simultaneously presents the hardest theoretical challenge for SFT.

\section{Dynamics of Online RL}
\label{sec:online_RL}

We now extend our analysis from the static setting to the dynamics of learning.
We consider an \textit{iterative on-policy} setting: at each iteration $i$, the agent collects a fresh dataset using the current policy $\pi^{(i)}$ and updates it to $\pi^{(i+1)}$.
This protocol, which discards historical samples, allows us to rigorously track the evolution of the relative budget $\xi_i$ and quantify the sample complexity required for monotonic policy improvement across different regimes.

\subsection{Sample Complexity for Monotonic Improvement}

Let $m \in \doubleZ_+$ be the total number of iterations.
We initialize the process with the base policy $\pi^{(0)} \coloneqq \pi_b$ and consider the learning dynamics over iterations $i = \{0, 1, \dots, m-1\}$.
For any instance $x$, we denote the expected time-to-solution and the instance-specific relative budget under the current policy $\pi^{(i)}$ as
$\mu_x^{(i)} \coloneqq \doubleE_{\tau \sim \pi^{(i)}(\cdot\mid x)}[T(\tau)]$
and
$\xi_x^{(i)} \coloneqq H/\mu_x^{(i)}$, respectively.
Fix a target relative budget $\xi>0$ and define the corresponding slice of inputs
$\Xcal(\xi) \coloneqq \{x: \xi/2 \le \xi_x^{(i)} \le 2\xi\}$
together with the slice-conditioned input distribution
$\rho_\xi(x) \coloneqq \rho (x \mid x\in \Xcal(\xi))$.
In the following analysis, expectations are taken over $x \sim \rho_\xi$ and $\tau \sim \pi^{(i)}(\cdot \mid x)$ unless stated otherwise.
Accordingly, we define the representative relative budget for this slice as
\begin{equation}
    \xi_{i} \coloneqq H / \doubleE_{x \sim \rho_\xi}[\mu_x^{(i)}].
\end{equation}
To ensure the validity of our aggregate analysis, we introduce a homogeneity assumption for problems within the same relative budget slice. This implies that the reward variance is primarily determined by the relative budget regime and is approximately constant across inputs within the slice.

\begin{assumption}[Homogeneity within relative budget slice]
\label{assump:homogeneity}
For each iteration $i$, we assume reward homogeneity within $\Xcal(\xi)$, implying that for all $x \in \Xcal(\xi)$, the conditional variance is approximated by the slice-level statistic.
\end{assumption}

Accordingly, we define the representative standard deviation of the cumulative reward achieved by the current policy $\pi^{(i)}$ as the mean conditional standard deviation: $\sigma(\pi^{(i)}) \coloneqq \doubleE_{x\sim\rho_\xi}\left[\sqrt{
\doubleV_{\tau\sim\pi^{(i)}(\cdot\mid x)}[R(\tau)]}\right]$.

Building on this formalism, we now establish a finite-sample guarantee for a single iteration.
The following theorem characterizes the policy improvement by updating from $\pi^{(i)}$ to $\hat{\pi}^{(i+1)}$ using $n_i$ sample rollouts drawn from $\pi^{(i)}$.

\begin{theorem}[One-step improvement by RL]
\label{thm:one-step-RL}
Fix an iteration index $i$.
Let $\pi^{(i)}$ denote the current base policy with relative budget $\xi_{i}$, and let $\kappa_i$ be the trust region radius (cf. \cref{def:anticonc}).
Under Assumption~\ref{assump:homogeneity}, let $\hat{\pi}^{(i+1)}$ be the policy obtained by the verifier-based RL update specified in \cref{thm:upper} with a trust region radius $\kappa_i$, using $n_{i}$ sample rollouts from $\pi^{(i)}$.
Then, with probability at least $1-\delta$,
\begin{equation*}
    J_{\xi_i}\big(\hat{\pi}^{(i+1)}\big)
    \ge J_{\xi_i}\big(\pi^{(i)}\big)
    + \sqrt{\kappa_{i}}\,\sigma(\pi^{(i)})
    - \frac{CH\log(|\Rcal|/\delta)}{c_0(\xi_{i}; \kappa_{i})\,n_i}.
\end{equation*}
Moreover, if the sample size $n_i$ is sufficiently large such that
\begin{equation}
  n_i \ge \frac{2CH\log(|\mathcal{R}|/\delta)}{\sqrt{\kappa_i} \, c_0(\xi_{i}; \kappa_{i}) \, \sigma(\pi^{(i)})},
  \label{eq:estimator-n}
\end{equation}
then setting the next base policy to $\pi^{(i+1)} \coloneqq \hat{\pi}^{(i+1)}$ yields
\begin{equation}
    J_{\xi_i} \big(\pi^{(i+1)}\big) \ge J_{\xi_i}\big(\pi^{(i)}\big)
    + \frac{1}{2}\sqrt{\kappa_i}\,\sigma(\pi^{(i)}).
  \label{eq:estimator-improve}
\end{equation}
Here, $C$ is a universal constant implied by Lemma~\ref{thm:upper}, and $c_0(\xi_{i}; \kappa_{i})$ is the anti-concentration coefficient for $\pi^{(i)}$.
\end{theorem}
The proof relies on decomposing the performance improvement into two competing terms: a theoretical improvement component and a stochastic estimation error.
The theorem is established by determining the minimum sample size $n_i$ required for the improvement signal to dominate the error, thereby guaranteeing a monotonic increase in the objective.
The full proof is detailed in \cref{proof_thm:one-step-RL}.

\subsection{Role of Relative Budget in Learning Dynamics}

Applying the finite-sample guarantee from \cref{thm:one-step-RL} to the regimes defined in \cref{sec:rl}, we derive the minimum sample complexity $n_i$ for monotonic policy improvement. 
Specifically, substituting the scaling laws for reward standard deviation $\sigma$ and anti-concentration $c_{0}(\xi; \kappa_i)$ yields the following regime-dependent complexity bounds.

\begin{theorem}[Three regimes of relative budget in online RL]
\label{thm:three-regimes_onlineRL}
Consider an online RL algorithm satisfying the one-step improvement guarantee of \cref{thm:one-step-RL}.
Then the attainable per-iteration improvement and the sufficient sample complexity $n_i$ required for monotonic policy improvement exhibit three distinct relative budget regimes:

\begin{enumerate}[itemsep=0pt, parsep=0pt, topsep=1pt, partopsep=0pt]
    \item (Deficient: $\xi_i \ll 1$)
        The improvement is bounded by $\mathcal{O}(H \sqrt{\kappa_i f(\xi_i)})$, with diverging sample complexity:
        \begin{align}
            n_i = \widetilde{\Omega}
                \left(
                    \kappa_{i}^{-1/2} (f(\xi_{i}))^{-3/2}
                \right).
        \end{align}
    \item (Balanced: $\xi_i = \Theta(1)$)
        The improvement is
        $\mathcal{O}(\sqrt{\kappa_i} H)$,
        and the sufficient sample size is
        \begin{align}
            n_i = \widetilde{\Omega}\left(1/\sqrt{\kappa_{i}}\right).
        \end{align}
    \item (Ample: $\xi_i \gg 1$)
        The improvement decays to
        $\mathcal{O}(\sqrt{\kappa_i} H / \xi_i)$, and the sample complexity scales as
        \begin{align}
            n_i = \widetilde{\Omega}\left(\xi_{i}/\sqrt{\kappa_{i}}\right).
        \end{align}
\end{enumerate}
\end{theorem}
Learning is most sample-efficient in the balanced regime, while both insufficient and excessive relative budgets lead to reduced improvement efficiency, either through vanishing learning signal or increasing sample cost.
Note that, in the ample regime, ensuring monotonic improvement typically requires scaling the trust region as $\sqrt{\kappa_i} = \Ocal(\xi_i^{-1})$, which results in an effective improvement rate of $\Ocal(H/\xi_i^2)$ and a quadratic degradation in sample complexity: $n_i = \widetilde{\Omega}(\xi_{i}^2)$.

\subsection{Linear Budget Growth under Gamma Model}
\label{sec:online_RL_gamma}

Finally, we adopt the same assumption as \cref{sec:gamma_statistics} where the time-to-solution of a policy follows a gamma distribution.
\cref{thm:online_rl} guarantees under this model that the relative budget $\xi_i$ grows linearly over iterations, while the expected return $J(\pi^{(i)})$ converges asymptotically to the token budget $H$.\footnote{As the gamma model assumes $T(\tau) \mid x \sim \mathsf{Gamma}(K,p)$ for any $x\in\mathcal{X}$, we draw input $x$ from the entire problem set $\mathcal{X}$, not from a certain input slice (thereby dropping the subscript of $J$).}
\Cref{thm:online_rl} also demonstrates the fundamental trade-off in the ample budget regime: the linear growth of the relative budget comes at the cost of quadratic sample complexity per iteration.

\begin{theorem}[Linear relative budget growth under Gamma model]
  \label{thm:online_rl}
  Let $m \in \{1,2,\ldots\}$ and $\delta \in (0, 1]$ be arbitrary.
  Under the gamma model and by repeating the iteration in \cref{thm:one-step-RL} with 
  $\kappa_i \coloneqq \left(\min\left\{\frac{\sigma(\pi^{(i)})}{J(\pi^{(i)})}, \frac{H - J(\pi^{(i)})}{2\,\sigma(\pi^{(i)})}\right\}\right)^2$,
  it holds with probability at least $1 - \delta$ that, for each $i \in [m]$,
  \begin{align}
    \xi_i - \xi_0 & \ge i \, (2K)^{-1} - \Ocal_{K}(1), \\
    J(\pi^{(i)}) & = H \left(1 - \xi_i^{-1} + \Ocal_K\!\left(\xi_i^{K-1} e^{-K\xi_i}\right)\right), \\
    \sigma^2(\pi^{(i)}) & = H^2 \left((\xi^2_{i} K)^{-1} + \Ocal_K\!\left(\xi_i^{K-1} e^{-K\xi_i}\right)\right), \\
    c_0^{(i)} & = \Theta_K\!\left(1\right),
  \end{align}
  provided that the number of rollouts $n_i$ from $\pi^{(i)}$ satisfies
  \begin{align}
    n_i
    & \ge \frac{2\,C H \log \left(m \left|\Rcal \right| / \delta \right)}{\sqrt{\kappa_{i}}\, c_0^{(i)} \sigma(\pi^{(i)})} 
    = \widetilde\Theta\left(K \xi_{i}^2\right).
  \end{align}
  Here, $C$ is a universal constant implied by Lemma~\ref{thm:upper}, and $c_0^{(i)}$ is the anti-concentration coefficient of $\pi^{(i)}$, defined as $c_0^{(i)} \coloneqq \doubleP_{x\sim\rho, \tau\sim\pi^{(i)}(\cdot \mid x)}[R(\tau) \ge J(\pi^{(i)}) + \sqrt{\kappa_i}\, \sigma(\pi^{(i)})]$.
\end{theorem}

The proof in \cref{proof_thm:online_rl} is based on an asymptotic evaluation using the per-step guarantee by \cref{thm:one-step-RL} and the statistics analysis in \cref{sec:gamma_statistics}.

\begin{remark}
    Under the gamma model, online RL can improve any base policy so that the mean time-to-solution becomes arbitrarily small.
    While this may appear unrealistic, each iteration requires an unbounded number of (i.e., $\widetilde\Omega_K(i^2)$ many) sample rollouts from the current policy.
    Additionally, the continuous approximation by the gamma distribution may involve a significant gap compared to the negative binomial distribution in the ample budget regime.
\end{remark}

\section{Experiments}
\label{sec:experiment}

\subsection{Empirical Validation of Relative Budget Theory}
\label{subsec:empirical_validation}

To validate the regime-dependent behaviors predicted in \cref{thm:rl_suboptimality}, we evaluated \texttt{Llama-3.2-3B-Instruct} \cite{grattafiori2024llama}, \texttt{Phi4-mini-instruct} \cite{abouelenin2025phi}, and \texttt{Qwen3-4B-Instruct} \cite{yang2025qwen3} on GSM8K~\cite{cobbe2021gsm8k} and MATH-500~\cite{hendrycks2measuring,lightman2023lets}. 
We measured the reward variance and anti-concentration coefficient $c_0$ across varying relative budgets $\xi = H/\mu_x$.
See Appendix~\ref{app:static_analysis_details} for detailed setup and additional results.

\Cref{fig:llama_anti_concentration} shows representative results on GSM8K with \texttt{Llama-3.2-3B-Instruct}, which are consistent with our theoretical framework (see \cref{fig:c0_eps_rl_1,fig:c0_eps_rl_2}).
We observed consistent regime-dependent behaviors across all evaluated models and datasets (see \cref{app:static_analysis_details} for complete plots).
First, we observe a phase transition around $\xi \approx 1.0$. 
In the deficient regime ($\xi < 0.8$), both metrics remain negligible, indicating that the learning signal effectively vanishes.
While the anti-concentration coefficient $c_0$ increases and peaks around $\xi \approx 1.2$, the reward variance reaches a plateau at a larger relative budget, approximately $1.5 \le \xi \le 1.8$ in our experiments.
This suggests that although high-reward solutions become discoverable once the task enters the solvable regime (around $\xi \approx 1.2$), maximizing the learning signal requires a larger buffer ($H \approx 1.5\mu_x \text{--} 1.8\mu_x$) to accommodate the long-tailed behavior of the generation distribution.
Finally, for $\xi > 2.0$, the anti-concentration coefficient saturates and remains stable at a high level, supporting the robustness of RL in ample-budget regimes.
  
\begin{figure}
    \centering
    \includegraphics[width=0.97\linewidth]{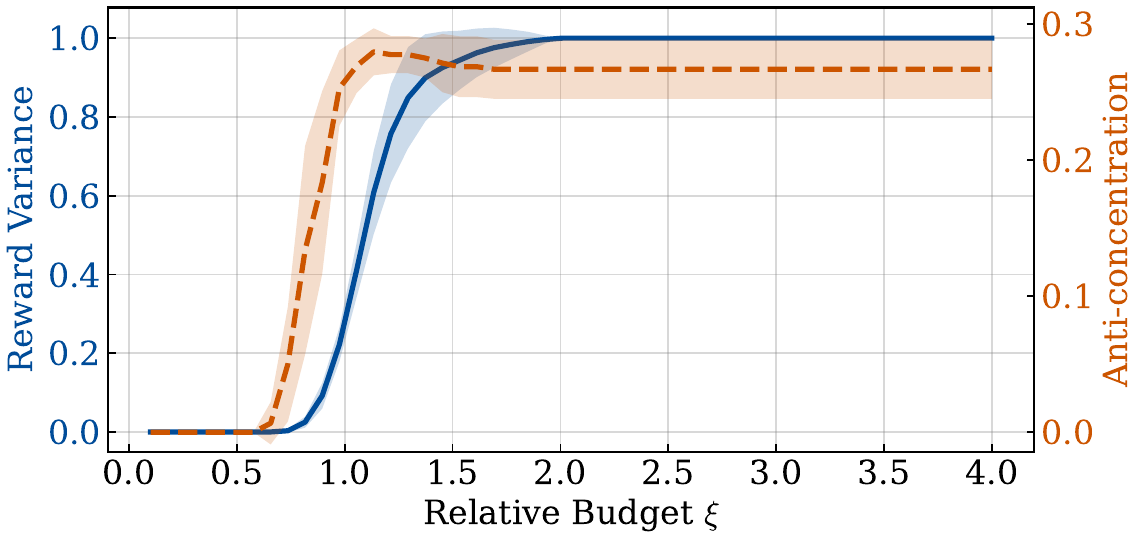}
    \caption{Reward variance and anti-concentration coefficient vs. relative budget $\xi$ for \texttt{Llama-3.2-3B-Instruct} on GSM8K.
    A phase transition occurs at $\xi \approx 1.0$, with the learning signal maximizing at $\xi \approx 1.5 \text{--} 1.8$ in our experiments.}
    \label{fig:llama_anti_concentration}
\end{figure}

\subsection{Performance Improvement via RLVR}
\label{sec:experiment_rlvr}

To validate the learning dynamics predicted in \Cref{thm:rl_suboptimality}, we fine-tuned \texttt{Llama-3.2-3B-Instruct} and \texttt{Qwen3-4B-Instruct} using Group Relative Policy Optimization (GRPO, \citet{shao2024deepseekmath}).
We measured post-training accuracy on two reasoning tasks.
First, we performed both training and evaluation on GSM8K.
Second, to assess performance on more challenging tasks, we trained models on DAPO-Math-17k-Processed~\cite{yu2025dapo} and evaluated them on the MATH-500 test set.
We define the token budget $H$ as the maximum completion length.
Holding the task distribution fixed, we varied the token budget $H$ to traverse different relative budget regimes $\xi \coloneqq H/\mu$ (with $\mu$ estimated under $\pi_b$).
We trained separate models for a fixed number of steps at each $H$ and measured test-set accuracy.

\Cref{fig:rlvr_experiment} illustrates the relationship between the relative budget $\xi$ and the post-training test accuracy across models and datasets. 
Consistent with our theoretical framework, the results exhibit a phase-transition behavior governed by the relative budget.
In the \textit{deficient} regime ($\xi \ll 1$), the accuracy remains negligible, which aligns with our analysis that informative trajectories are exceedingly rare and the model fails to learn effective reasoning paths.
In the \textit{balanced} regime, as $\xi$ approaches $1$, we observe a sharp performance improvement.
This corresponds to the phase transition point where the compute budget $H$ becomes comparable to the inherent task difficulty $\mu_x$, making high-reward solutions discoverable and enabling efficient learning.
In the \textit{ample} regime beyond $\xi \approx 1.5$, performance gains saturate: marginal benefits diminish because tasks are already well within model capabilities. 
These consistent findings support the relative budget $\xi$ as a key factor for RLVR efficacy.

\begin{figure}[t]
    \centering
    \includegraphics[width=\linewidth]{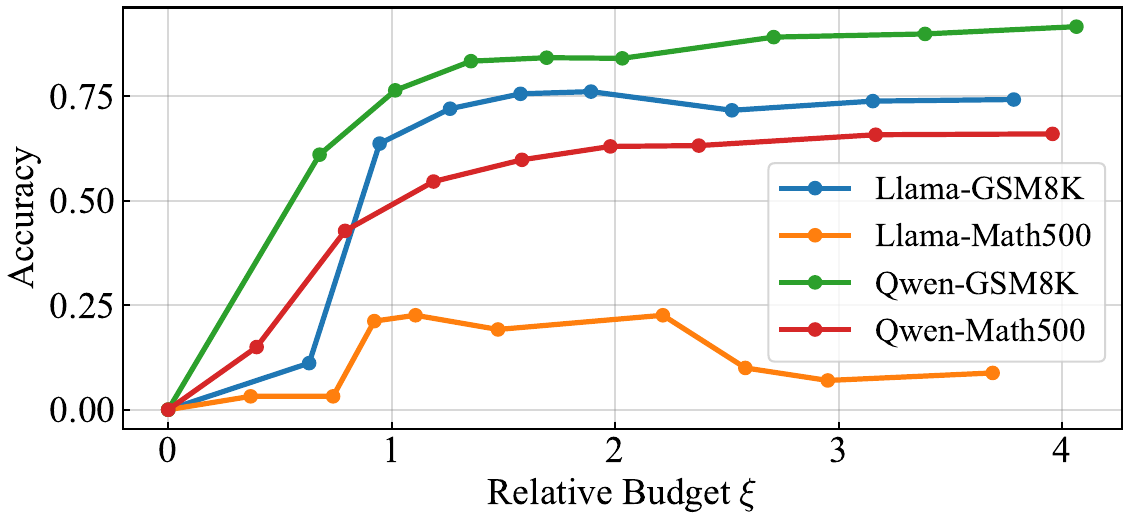}
    \caption{The relationship between relative budget $\xi$ and accuracy in LLMs trained via RLVR.}
    \label{fig:rlvr_experiment}
\end{figure}

\section{Conclusion and Future Work}
\label{sec:conclusion}

We have proposed a relative-budget theory unifying compute constraints and task difficulty to analyze the sample complexity of verifier-based RL.
Our analysis reveals that learning efficiency depends on the anti-concentration of the time-to-solution distribution, which identifies three regimes based on the relative budget $\xi$.
A phase transition emerges at the balanced regime ($\xi \approx 1$): this point provides the strongest learning signal for RL, yet is the most challenging for SFT due to peak solution heterogeneity. 
Our empirical results support these dynamics, identifying an optimal budget of $\xi \in [1.5, 2.0]$ consistent with our theory.
We further proved that online RL drives linear relative budget growth, offering a theoretical basis for test-time scaling.
Practically, our findings suggest allocating compute slightly above the expected difficulty ($\xi \gtrsim 1$) to ensure reliable learning signals while avoiding the inefficiency of excessive budgets.

While this paper provides a theoretical foundation for verifier-based RL, our analysis relies on a continuous proxy reward to make the variance analysis tractable.
Extending the theory to the binary-reward setting (i.e., terminal 1/0 outcome supervision) remains a key open problem for characterizing standard reasoning pipelines and for developing a more general theory of reasoning.

\section*{Acknowledgement}
TS and RH were partially supported by JSPS KAKENHI (24K02905) and JST CREST (JPMJCR2115). This research is supported by the National Research Foundation, Singapore, Infocomm Media Development Authority under its Trust Tech Funding Initiative, and the Ministry of Digital Development and Information under the AI Visiting Professorship Programme (award number AIVP-2024-004). Any opinions, findings and conclusions or recommendations expressed in this material are those of the author(s) and do not reflect the views of National Research Foundation, Singapore, Infocomm Media Development Authority, and the Ministry of Digital Development and Information.

\section*{Impact Statement}

This research contributes to more sustainable AI development by providing a framework to identify when reinforcement learning for reasoning is likely to yield diminishing returns.
By offering guidance on when RL for reasoning is likely to be productive versus wasteful, this work enables research and development teams to avoid unnecessary training runs and better target limited compute budgets.
The insights in this paper could lower the financial costs and environmental footprint associated with large-scale model training. However, while this contribution is theoretical, a plausible negative outcome is that organizations might use these results to justify increased aggregate compute spending to chase performance gains, potentially increasing emissions.

\bibliography{ref}
\bibliographystyle{icml2026}

\newpage
\appendix
\onecolumn

\section{Token Distributions}
\label{appendix:token_dist}

To validate the theoretical assumption made in \cref{sec:gamma_statistics,sec:online_RL_gamma} that reasoning steps follow a gamma distribution, we analyzed the token counts generated by three popular LLMs.

\Cref{fig:fig_token_distribution} visualizes the distribution of solution lengths for \texttt{Llama-3.2-3B-Instruct} and \texttt{Phi-4-mini-instruct} on the GSM8K~\cite{cobbe2021gsm8k} and MATH-500~\cite{hendrycks2measuring,lightman2023lets} datasets.
In each subplot, the blue histogram represents the empirical distribution of generated tokens, while the red curve shows the fitted gamma distribution. 
The close alignment between the observed data and the fitted curves empirically supports our modeling choice, demonstrating that the gamma distribution serves as a reliable continuous proxy for the discrete token counts in complex reasoning tasks.

\begin{figure}[h]
  \centering
  \begin{subfigure}[b]{0.45\textwidth}
    \centering
    \includegraphics[height=4cm]{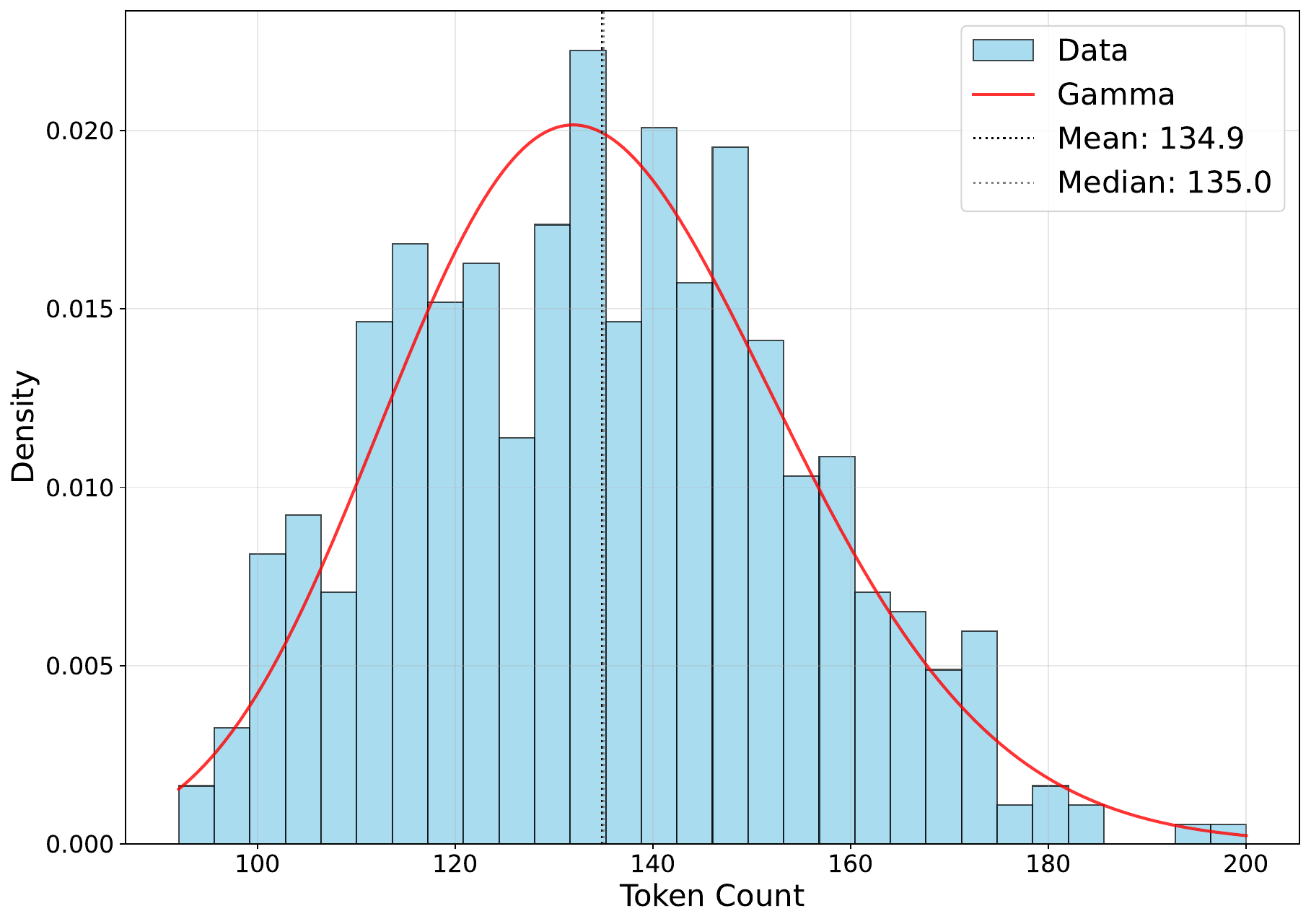}
    \caption{\texttt{Llama-3.2-3B-Instruct} + GSM8K.}
    \label{fig:llama_gsm8k}
  \end{subfigure}
  \begin{subfigure}[b]{0.45\textwidth}
    \centering
    \includegraphics[height=4cm]{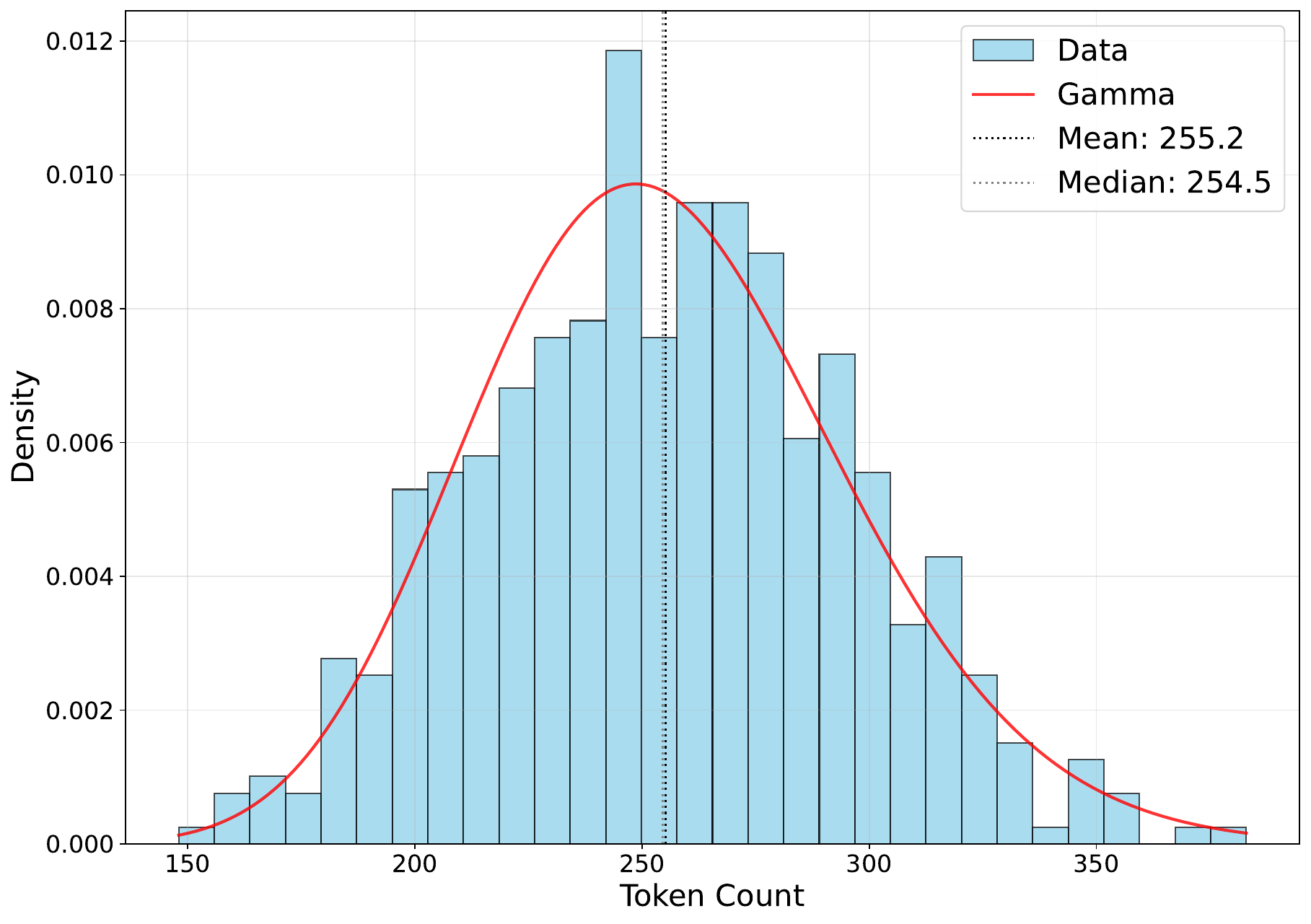}
    \caption{\texttt{Llama-3.2-3B-Instruct} + MATH500.}
    \label{fig:llama_math500}
  \end{subfigure}
  \par
  \begin{subfigure}[b]{0.45\textwidth}
    \centering
    \includegraphics[height=4cm]{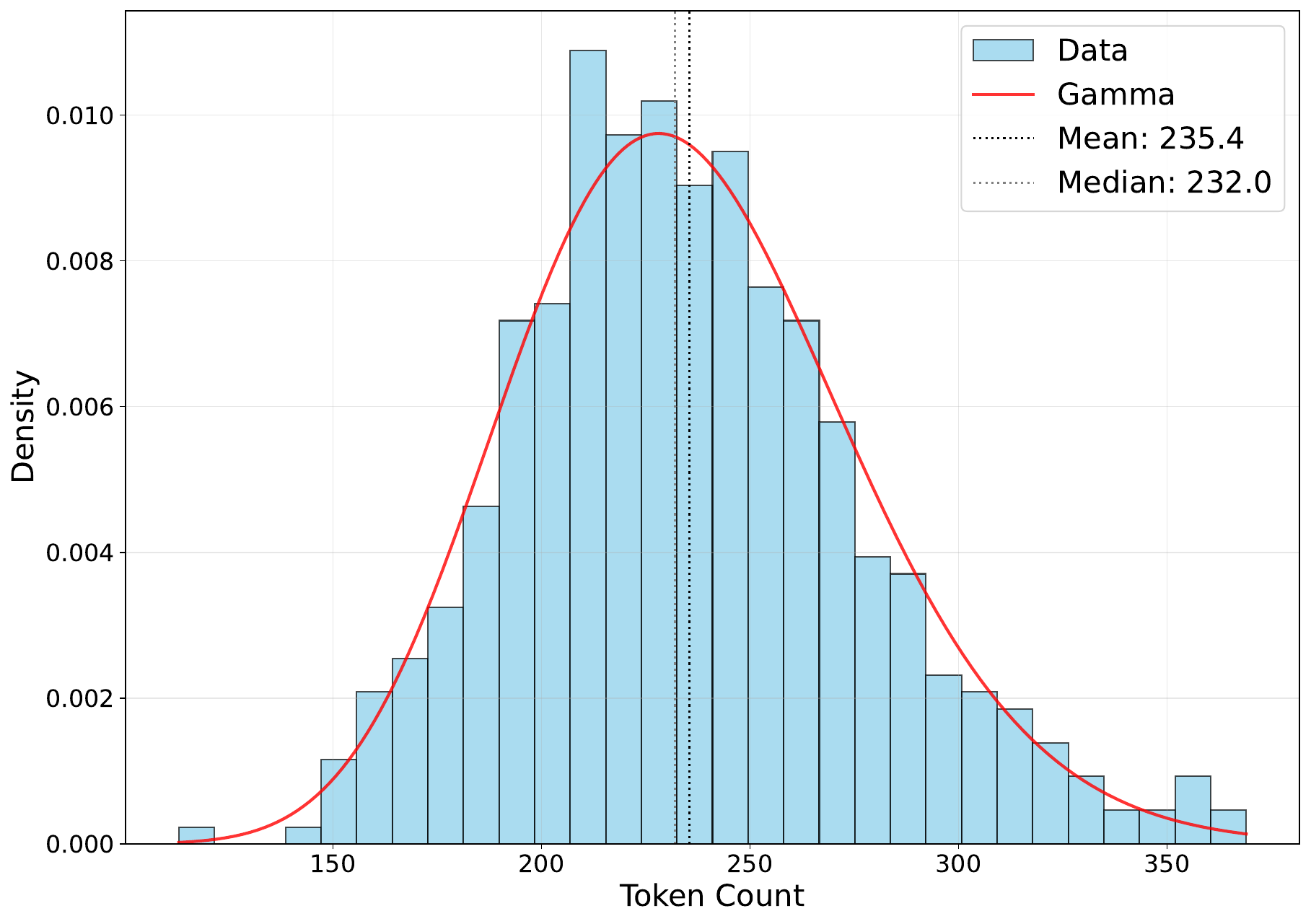}
    \caption{\texttt{Phi-4-mini-instruct} + GSM8K.}
    \label{fig:phi4_gsm8k}
  \end{subfigure}
  \begin{subfigure}[b]{0.45\textwidth}
    \centering
    \includegraphics[height=4cm]{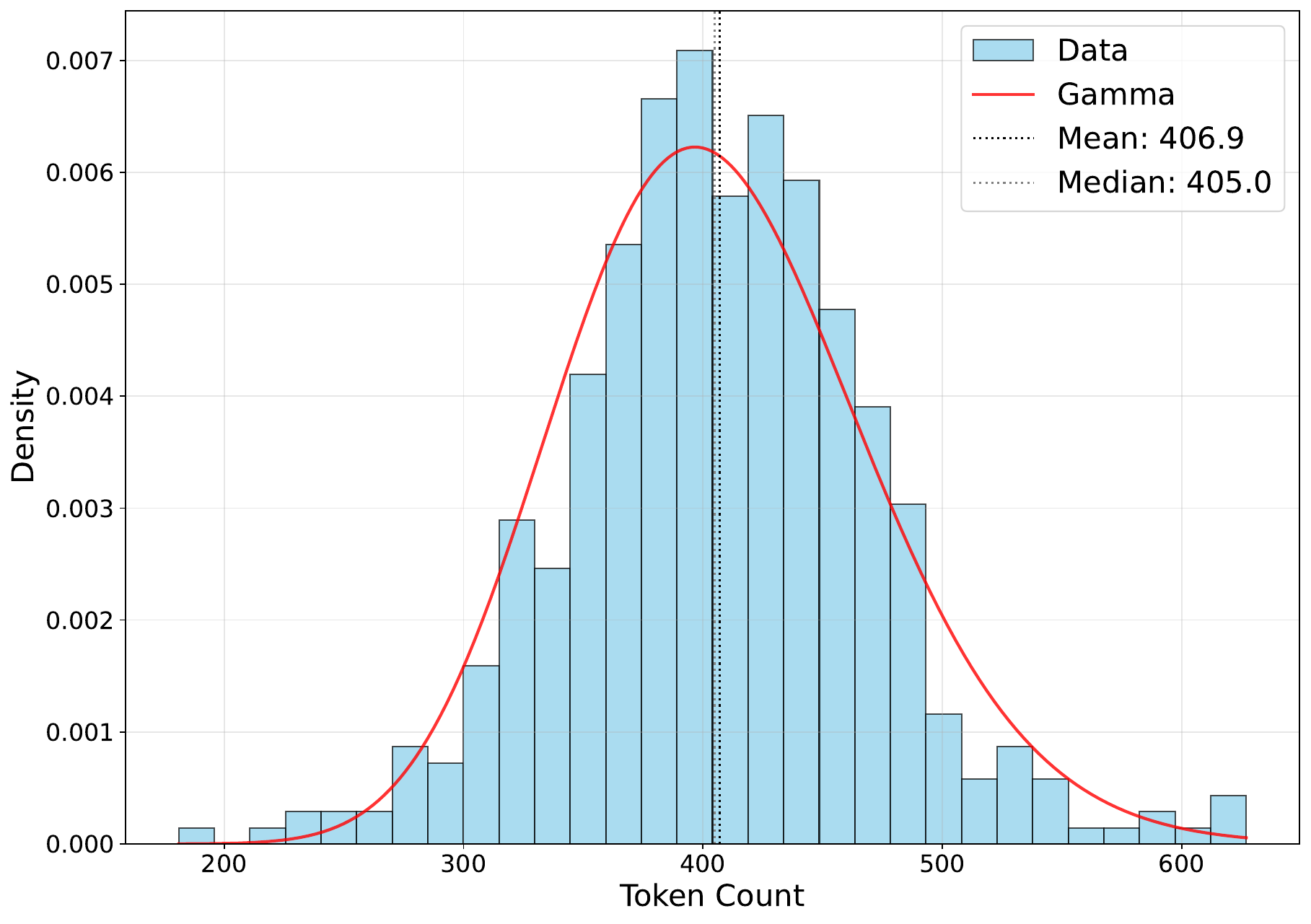}
    \caption{\texttt{Phi-4-mini-instruct} + MATH500.}
    \label{fig:phi4_math500}
  \end{subfigure}
  \caption{Token distributions generated by LLMs for a problem instance in reasoning datasets and their gamma distribution fittings.}
   \label{fig:fig_token_distribution}
\end{figure}

However, models like \texttt{Qwen3-4B-Instruct} exhibit a bimodal distribution as shown below. 
Recent studies attribute such behavior to data contamination, where the model oscillates between retrieving memorized solutions and engaging in genuine reasoning~\cite{wu2025reasoning}. 
While our gamma-based theory models the latter, these exceptions would require a mixture model approach.
Note that our core theoretical framework still holds for such outlier models as discussed in \cref{app:static_analysis_details}.
\begin{figure}[h]
  \centering
  \begin{subfigure}[b]{0.45\textwidth}
    \centering
    \includegraphics[height=4cm]{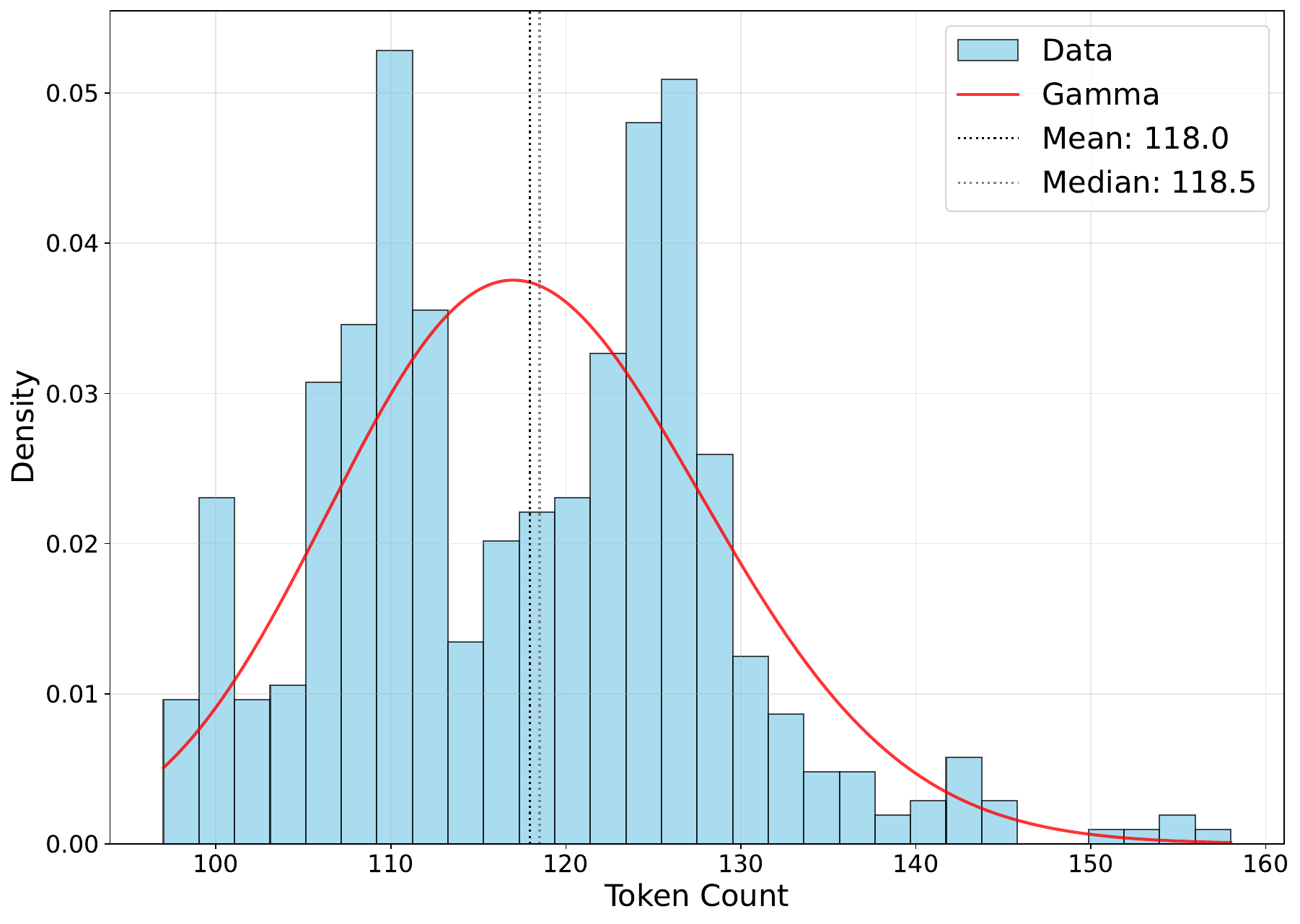}
    \caption{\texttt{Qwen3-4B-Instruct} + GSM8K.}
    \label{fig:qwen_gsm8k}
  \end{subfigure}
  \begin{subfigure}[b]{0.45\textwidth}
    \centering
    \includegraphics[height=4cm]{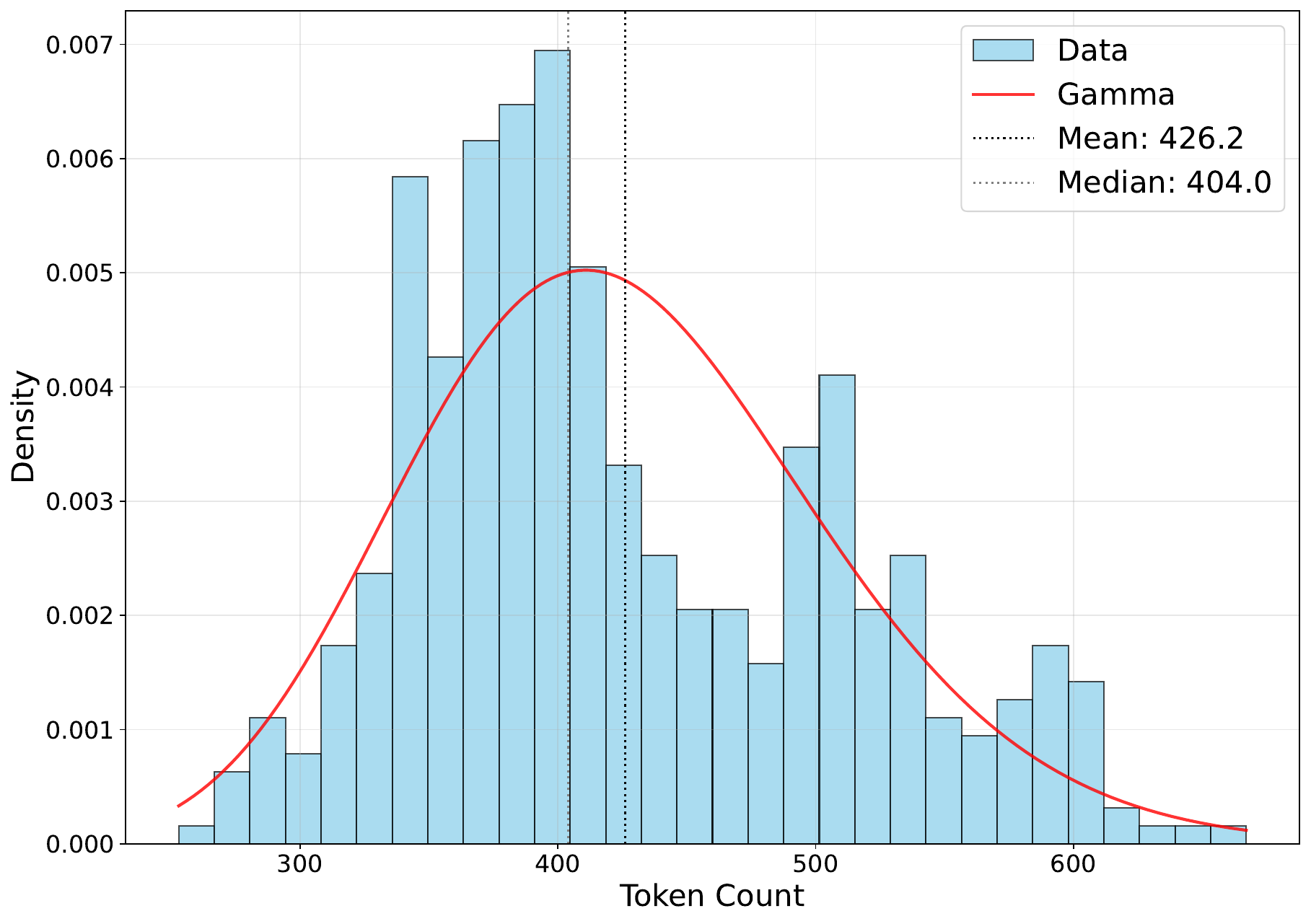}
    \caption{\texttt{Qwen3-4B-Instruct} + MATH500.}
    \label{fig:qwen_math500}
  \end{subfigure}
  \label{fig:token_distribution_qwen}
  \caption{Token distributions generated by \texttt{Qwen3-4B-Instruct}. The empirical data exhibit a bimodal structure, showing a clear deviation from the single gamma distribution fit.}
\end{figure}

\section{Experimental Details for Static Analysis}
\label{app:static_analysis_details}

In Section~\ref{subsec:empirical_validation}, we performed empirical analyses to validate our relative budget theory.
The detailed experimental setup is as follows.
For the data collection, we randomly sampled $20$ problems from the GSM8K and MATH-500 test sets. 
For each problem, we generated $100$ solution traces using \texttt{Llama-3.2-3B-Instruct}, \texttt{Phi-4-mini-instruct}, and \texttt{Qwen3-4B-Instruct} with temperature $T=0.7$ and $\texttt{top\_p}=0.9$.
To estimate the true difficulty, we set a sufficiently large max completion length ($1024$ tokens) and calculated the mean solution length $\mu_x$ based on correct trajectories.
We then simulated varying budget constraints within $\xi \in [0.1, 4.0]$.
The shaped reward was computed as $R(\tau) = \max \{0, H - T(\tau) + 1\}$ for successful traces and $0$ otherwise.
To ensure statistical reliability, we aggregated the normalized variance curves across problems.

In addition to \Cref{fig:llama_anti_concentration} for \texttt{Llama-3.2-3B-Instruct} on GSM8K, \Cref{fig:llama_math500_c0,fig:phi4_anti_concentration_appendix} present additional experimental results that use \texttt{Llama-3.2-3B-Instruct} and \texttt{Phi-4-mini-instruct}.
Observe that the following figures exhibit similar tendencies to \cref{fig:llama_anti_concentration,}.

\begin{figure}[h]
    \centering
    \includegraphics[width=0.40\linewidth]{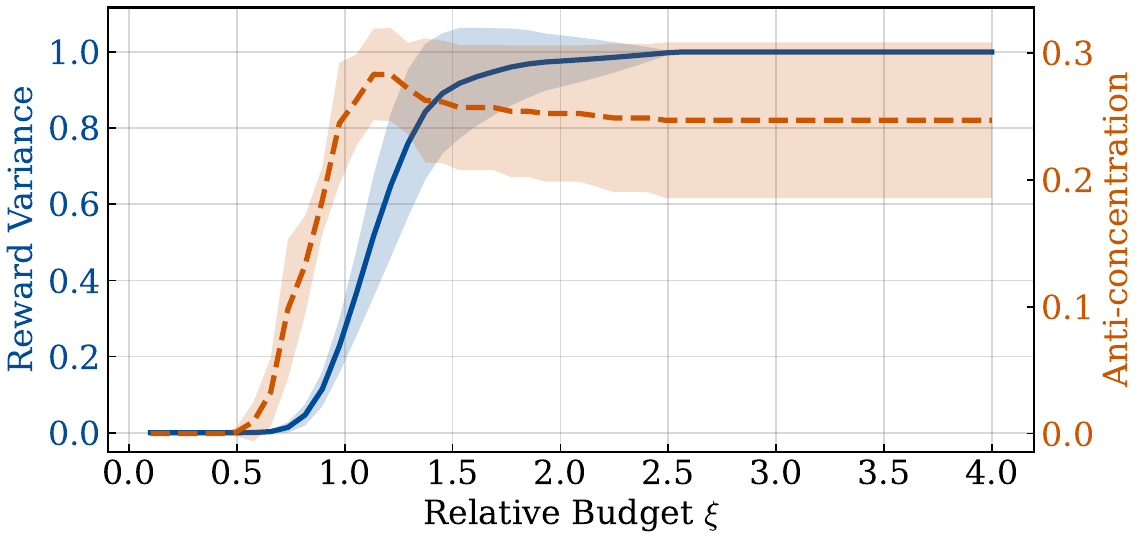}
        \caption{\texttt{Llama-3.2-3B-Instruct} + MATH-500.}
    \label{fig:llama_math500_c0}
\end{figure}

\begin{figure}[h]
    \centering
    \begin{subfigure}[t]{0.4\textwidth}
        \includegraphics[width=\linewidth]{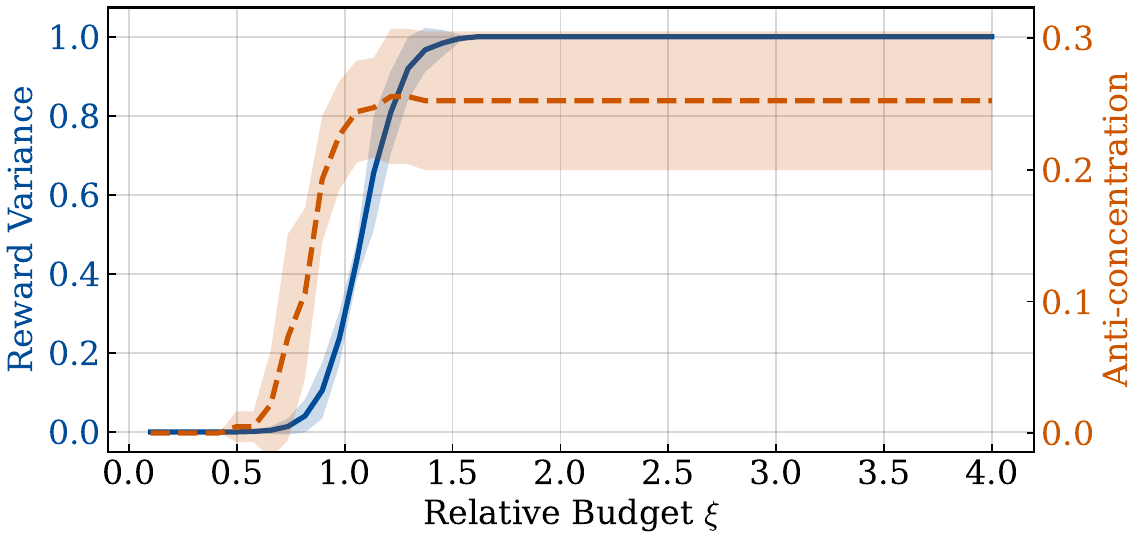}
        \caption{\texttt{Phi-4-mini-instruct} + GSM8K.}
    \end{subfigure}
    \begin{subfigure}[t]{0.4\textwidth}
        \includegraphics[width=\linewidth]{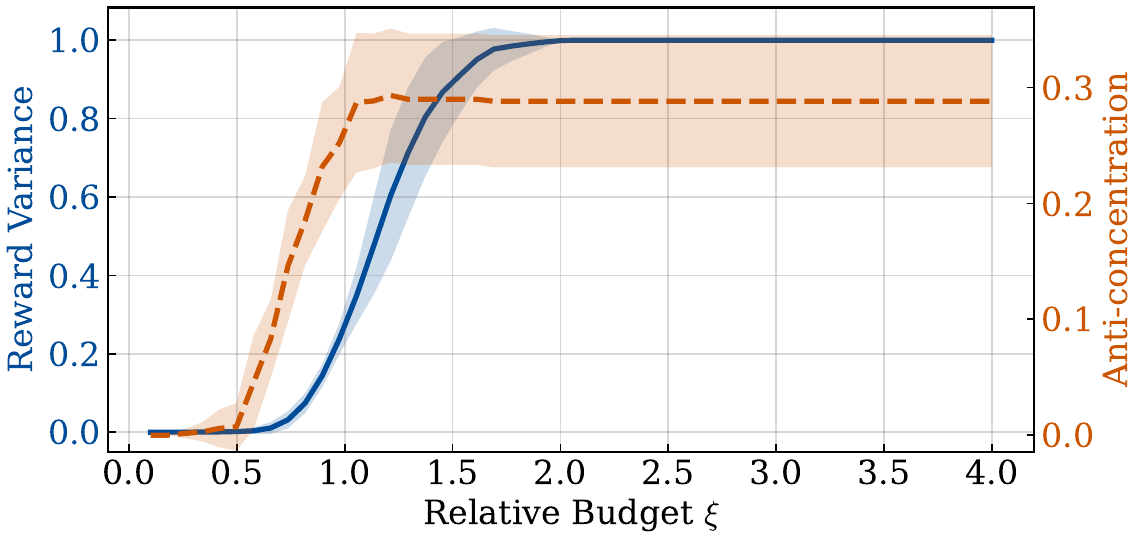}
        \caption{\texttt{Phi-4-mini-instruct} + MATH-500.}
    \end{subfigure}
    \caption{Phase transition in relative budget regimes. Normalized reward variance (blue) and anti-concentration coefficient (orange) vs. relative budget $\xi$. 
    A phase transition occurs at $\xi \approx 1.0$, with the learning signal maximizing at $\xi \approx 1.5 \text{--} 2.0$. in our experiments.}
    \label{fig:phi4_anti_concentration_appendix}
\end{figure}

We also present the following figures of the experimental results using \texttt{Qwen3-4B-Instruct} on GSM8K and MATH-500.
While token distributions generated by \texttt{Qwen3-4B-Instruct} are not aligned with the gamma distribution as discussed in \cref{appendix:token_dist}, we can see that the anti-concentration is as predicted in our theory.
\begin{figure}[h]
    \centering
    \begin{subfigure}[t]{0.4\textwidth}
        \includegraphics[width=\linewidth]{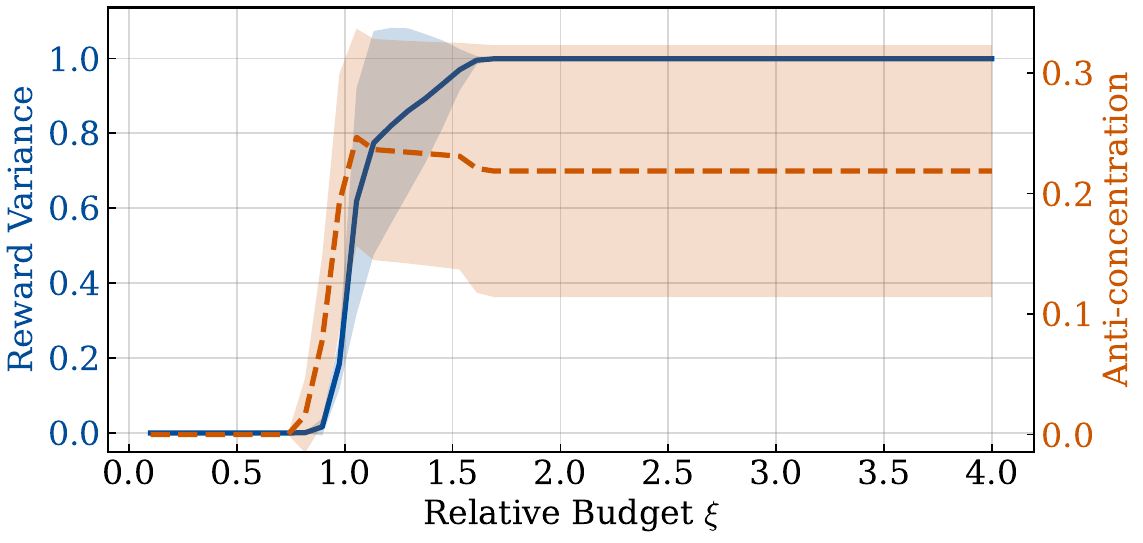}
        \caption{\texttt{Qwen3-4B-Instruct} + GSM8K.}
    \end{subfigure}
    \begin{subfigure}[t]{0.4\textwidth}
        \includegraphics[width=\linewidth]{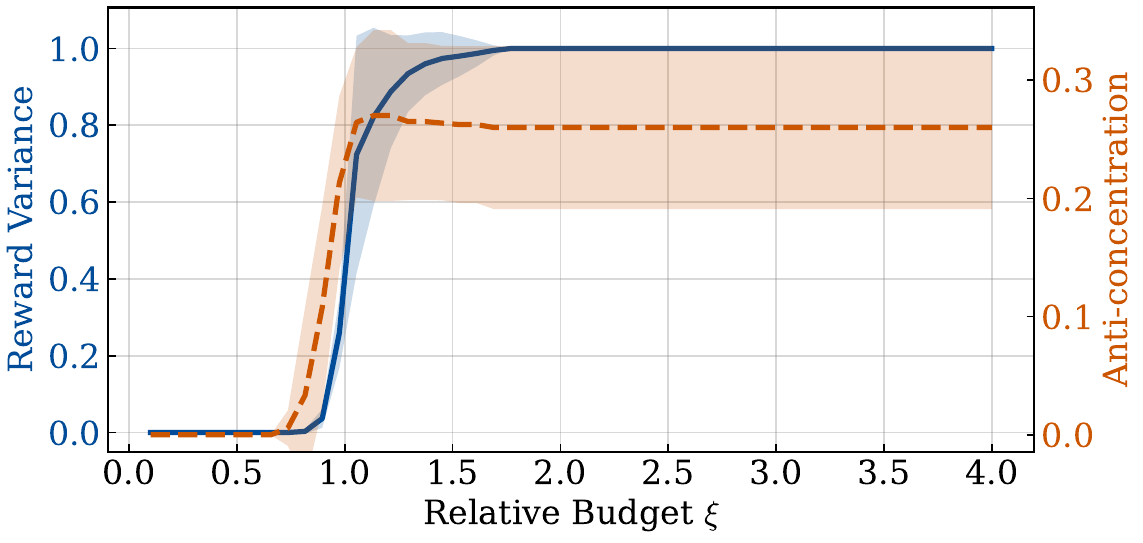}
        \caption{\texttt{Qwen3-4B-Instruct} + MATH-500.}
    \end{subfigure}
    \caption{Experimental results of \texttt{Qwen3-4B-Instruct}. As with other LLMs, we observe similar tendencies of the phase transition in relative budget regimes. Normalized reward variance (blue) and anti-concentration coefficient (orange) vs. relative budget $\xi$. 
    A phase transition occurs at $\xi \approx 1.0$, with the learning signal maximizing at $\xi \approx 1.5$ in our experiments.}
    \label{fig:qwen_anti_concentration_appendix}
\end{figure}

\section{Experimental Details for RL Fine-tuning}
\label{app:rl_details}

In \cref{sec:experiment_rlvr}, we perform RLVR using the GRPO algorithm.
This section details the hyper-parameters and implementation specifics.
We utilized a A100 GPU environment. 
The base models were \texttt{Llama-3.2-3B-Instruct} and \texttt{Qwen3-4B-Instruct}.
To manage memory efficiency, we employed 4-bit quantization during loading, although the LoRA adapters were trained in 16-bit precision.
We used the GRPOTrainer from the trl (\url{https://github.com/huggingface/trl}) library.
The configuration was set to explore the impact of the token budget $H$ on learning dynamics.
To validate the core claims in our relative budget theory presented in the main paper, we modulated the relative budget $\xi = H/\mu_x$ by varying the \texttt{max\_completion\_length} parameter in the GRPO configuration.
The specific hyper-parameters used for the results in \cref{sec:experiment_rlvr} are listed in \cref{tab:hyperparameters_llama,tab:hyperparameters_qwen}.

\begin{table}[h]
    \centering
    \caption{Hyper-parameters for GRPO Fine-tuning on \texttt{Llama-3.2-3B-Instruct}.}
    \label{tab:hyperparameters_llama}
    \begin{tabular}{lc}
    \toprule
    Hyperparameter & Value\\
    \midrule
    LoRA Rank & 64 \\
    LoRA Alpha & 64 \\
    Target Modules & q\_proj, k\_proj, v\_proj, o\_proj, \\
                   & gate\_proj, up\_proj, down\_proj \\
    \midrule
    Optimizer & AdamW 8-bit \\
    Learning Rate & $5 \times 10^{-6}$ \\
    LR Scheduler & Cosine (warmup ratio 0.1) \\
    Weight Decay & 0.1 \\
    Gradient Accumulation Steps & 4 \\
    Num Generations & 4 \\
    Max Gradient Norm & 1.0 \\
    Max Steps & 500 \\
    \bottomrule
    \end{tabular}
\end{table}

\begin{table}[h]
    \centering
    \caption{Hyper-parameters for GRPO Fine-tuning on \texttt{Qwen3-4B-Instruct}.}
    \label{tab:hyperparameters_qwen}
    \begin{tabular}{lc}
    \toprule
    Hyperparameter & Value \\
    \midrule
    LoRA Rank & 32 \\
    LoRA Alpha & 64 \\
    Target Modules & q\_proj, k\_proj, v\_proj, o\_proj, \\
                   & gate\_proj, up\_proj, down\_proj \\
    \midrule
    Optimizer & AdamW 8-bit \\
    Learning Rate & $5 \times 10^{-6}$ \\
    LR Scheduler & Linear \\
    Weight Decay & 0.01 \\
    Gradient Accumulation Steps & 4 \\
    Num Generations & 4 \\
    Max Gradient Norm & 1.0 \\
    Max Steps & 500 \\
    \bottomrule
    \end{tabular}
\end{table}

\section{Proofs from \cref{sec:rl}}

\subsection{Preliminary Lemma}

The following lemma expresses the reward moments in terms of the truncated
time-to-solution moments.

\begin{lemma}[Reward moments under bi-level feedback]
\label{lemma:reward_moments}
For any fixed problem instance $x \in \mathcal{X}$, let 
\begin{align}
    q_x \coloneqq \doubleP_{\tau \sim \pi_b(\cdot \mid x)}[T(\tau) \le H]
\end{align}
denote the success probability under the base policy. 
Let $\tilde{\mu}_x$ and $\tilde{v}_x$ denote the mean and variance of the time-to-solution $T(\tau)$ conditioned on success, defined as:
\begin{align}
    \tilde{\mu}_x &\coloneqq \doubleE_{\tau \sim \pi_b(\cdot \mid x)}[T(\tau) \mid T(\tau) \le H], \\
    \tilde{v}_x &\coloneqq \doubleV_{\tau \sim \pi_b(\cdot \mid x)}[T(\tau) \mid T(\tau) \le H].
\end{align}
Then, the mean and variance of $R(\tau)$ satisfy:
\begin{align}
    \doubleE_{\tau \sim \pi_b(\cdot \mid x)}[R(\tau)]
    &= q_x \bigl(H + 1 - \tilde{\mu}_x\bigr), \label{eq:ER_closed_form}\\ 
    \doubleV_{\tau \sim \pi_b(\cdot \mid x)}[R(\tau)]
    &= q_x \tilde{v}_x
     + q_x (1-q_x)\bigl(H + 1 - \tilde{\mu}_x\bigr)^2.
     \label{eq:VR_closed_form}
\end{align}
\end{lemma}

\begin{proof}
By definition, we have
\begin{align}
    R(\tau) = (H - T(\tau) + 1) \cdot \doubleI\{T(\tau)\le H\}.
\end{align}
Thus, the following chain of equations hold:
\begin{align}
    \doubleE_{\tau \sim \pi_b(\cdot \mid x)}[R(\tau)]
    &= \doubleE_{\tau \sim \pi_b(\cdot \mid x)}\bigl[(H - T(\tau) + 1) \doubleI \{T(\tau)\le H\}\bigr] \\
    &= q_x \doubleE_{\tau \sim \pi_b(\cdot \mid x)}[H - T(\tau) + 1 \mid T(\tau)\le H] \\
    &= q_x(H + 1 - \tilde{\mu}_x).
\end{align}
For the second moment, we have
\begin{align}
    \doubleE_{\tau \sim \pi_b(\cdot \mid x)}[R(\tau)^2]
    &= \doubleE_{\tau \sim \pi_b(\cdot \mid x)}\bigl[(H - T(\tau) + 1)^2 \cdot \doubleI \{T(\tau)\le H\}\bigr]\\
    &= q_x \doubleE_{\tau \sim \pi_b(\cdot \mid x)}\bigl[(H - T(\tau) + 1)^2\mid T(\tau)\le H\bigr]\\
    &= q_x\bigl(\tilde{v}_x + (H + 1 - \tilde{\mu}_x)^2\bigr).
\end{align}
Therefore, the following equation holds:
\begin{align}
    \doubleV_{\tau \sim \pi_b(\cdot \mid x)}[R(\tau)]
    &= \doubleE_{\tau \sim \pi_b(\cdot \mid x)}[R(\tau)^2] - \doubleE_{\tau \sim \pi_b(\cdot \mid x)}[R(\tau)]^2 \\
    &= q_x\tilde{v}_x + q_x(1-q_x)(H + 1 -\tilde{\mu}_x)^2.
\end{align}
\end{proof}

\subsection{Proof of \cref{lemma:anticonc_T_tail}}
\label{proof_lemma:anticonc_T_tail}

\begin{proof}
We analyze the event where the reward $R(\tau)$ exceeds its mean by
$\sqrt{\varepsilon}$ standard deviations. Using the bi-level reward definition
$R(\tau) = (H - T(\tau) + 1) \cdot \doubleI \{T(\tau) \le H\}$, the condition
\begin{align}
    R(\tau) \ge \doubleE_{\tau \sim \pi_b(\cdot \mid x)}[R(\tau)] + \sigma_{b,x}\sqrt{\varepsilon}
\end{align}
implies successful trajectories and transforms as follows:
\begin{align}
    H + 1 - T(\tau)
    & \ge \doubleE_{\tau \sim \pi_b(\cdot \mid x)}[R(\tau)] + \sigma_{b,x}\sqrt{\varepsilon} \\
    T(\tau)
    &\le H + 1 - \doubleE_{\tau \sim \pi_b(\cdot \mid x)}[R(\tau)] - \sigma_{b,x}\sqrt{\varepsilon}.
\label{eq:lemma41_proof_1}
\end{align}
By \cref{lemma:reward_moments}, the reward moments are functions of the conditional
time-to-solution moments
$\tilde\mu_x \coloneqq \doubleE_{\tau \sim \pi_b(\cdot \mid x)}[T(\tau)\mid T(\tau)\le H]$ and
$\tilde v_x \coloneqq \doubleV_{\tau \sim \pi_b(\cdot \mid x)}[T(\tau)\mid T(\tau)\le H]$.
In particular, \cref{lemma:reward_moments} shows that the right-hand side of~\eqref{eq:lemma41_proof_1} can be rewritten
as a deviation from the conditional mean.
We thus define the normalized threshold
\begin{equation*}
\eta(x,\varepsilon)
\coloneqq
\frac
    {\min\{H, H + 1 - \doubleE_{\tau \sim \pi_b(\cdot \mid x)}[R(\tau)] - \sigma_{b,x}\sqrt{\varepsilon}\}}
    {\doubleE_{\tau \sim \pi_b(\cdot \mid x)}[T(\tau)]}.
\end{equation*}
In addition, let
\begin{align}
    \bar{\eta}(x,\varepsilon) \coloneqq [\eta(x,\varepsilon)]_0^{z_0}.
\end{align}
Therefore, the following inequality holds:
\begin{align}
    \doubleP_{\tau \sim \pi_b(\cdot \mid x)} \left[T(\tau)\le \eta(x,\varepsilon)\,\doubleE_{\tau \sim \pi_b(\cdot \mid x)}[T(\tau)]\right]
    \ge
    \doubleP_{\tau \sim \pi_b(\cdot \mid x)} \left[T(\tau)\le \bar \eta(x,\varepsilon)\,\doubleE_{\tau \sim \pi_b(\cdot \mid x)}[T(\tau)]\right].
\end{align}
Therefore, by \cref{assumption:left_tail_T_tau},
\begin{align}
\doubleP_{\tau \sim \pi_b(\cdot \mid x)} \left[T(\tau)\le \bar \eta(x,\varepsilon)\,\doubleE_{\tau \sim \pi_b(\cdot \mid x)}[T(\tau)]\right]
\ge c_-\, f\big(\bar \eta(x,\varepsilon)\big).
\end{align}
\end{proof}

\subsection{Proof of \cref{thm:rl_suboptimality}}
\label{appendix:proof_thm:rl_suboptimality}

\begin{proof}
Fix $\xi>0$ and consider the slice $\Xcal(\xi)$.
Define $c_0(\xi; \kappa):=\inf_{x\in \Xcal(\xi)} c_x(\kappa)$, where $\kappa$ is the
trust-region radius in \cref{def:anticonc}.
For the following slice-wise bound, we additionally assume that on the slice $\Xcal(\xi)$ the comparator satisfies $\kappa_x \le \kappa$ for all $x\in\Xcal(\xi)$, which rules out degenerate solutions where the average $\chi^2$ constraint is satisfied by allocating large divergence to a small subset of instances.
Since $c_x(\varepsilon)$ is non-increasing in $\varepsilon$,
we have $c_x(\kappa_x)\ge c_x(\kappa)\ge c_0(\xi; \kappa)$ for all $x\in\Xcal(\xi)$.
Hence the base policy is $c_0(\xi;\kappa)$-anti-concentrated on $\Xcal(\xi)$.
Applying Lemma~3.2, with probability at least $1-\delta$,
\eqref{eq:suboptimality_c0} holds.
Finally, Propositions~\ref{prop:low_criticality},
\ref{prop:middle_criticality}, and~\ref{prop:high_criticality} give the stated regime-dependent behavior of $c_0(\xi; \kappa)$.
\end{proof}

\subsection{Deficient Relative Budget Regime}
\label{proof_prop:low_criticality}

In the regime of $\xi_x \to 0$, the compute budget $H$ is much smaller than the expected time-to-solution $\mu_x$.
Hence, we have $H \ll \mu_x$, and successful
solutions with $T(\tau) \le H$ are rare events.
In particular, by Assumption~\ref{assumption:left_tail_T_tau}, 
$q_x = \Theta\bigl(f(\xi_x)\bigr)$ as $\xi_x \to 0$.
The following proposition characterizes how this small success probability controls the reward variance and anti-concentration.

\begin{proposition}[Deficient relative budget]
\label{prop:low_criticality}
In the regime of $\xi_x \to 0$,
under Assumptions \ref{ass:moments} and \ref{assumption:left_tail_T_tau}, the RL reward variance and anti-concentration coefficient jointly satisfy
\begin{align}
    \sigma_{b,x}
    = \Ocal \bigl(H\sqrt{f(\xi_x)}\bigr)
    \quad \text{and} \quad
    c_0(\xi)
    = \Theta\bigl(f(\xi)\bigr).
    \label{eq:c0_low_criticality}
\end{align}
\end{proposition}

\begin{proof}
Since successful reasoning ($T(\tau)\le H$) is a rare event in the regime of $\xi \to 0$, we have $R(\tau)=0$ with high probability. Consequently, the reward variance is dominated by these rare success events.
Using~$q_x = \Theta\bigl(f(\xi_x)\bigr)$,
\begin{align*}
    \doubleE_{\tau \sim \pi_b(\cdot \mid x)}[R(\tau)^2]
    = \sum_{t=1}^H (H+1-t)^2 \doubleP_{\tau \sim \pi_b(\cdot \mid x)}[T(\tau) = t]
    \le \sum_{t=1}^H (H+1)^2 \doubleP_{\tau \sim \pi_b(\cdot \mid x)}[T(\tau) = t]
    = (H+1)^2 q_x.
\end{align*}
Since $\doubleE_{\tau \sim \pi_b(\cdot \mid x)}[R(\tau)]$ is negligible compared to the second moment in this regime, we have $\doubleV_{\tau \sim \pi_b(\cdot \mid x)}[R(\tau)] \asymp \doubleE_{\tau \sim \pi_b(\cdot \mid x)}[R(\tau)^2]$, which yields $\sigma_{b,x} = \Ocal \bigl(H\sqrt{f(\xi_x)}\bigr)$.

As for the anti-concentration coefficient, Lemma~\ref{lemma:anticonc_T_tail} states that there exists a coefficient $\eta(x,\varepsilon)$ such that
\begin{align}
    c_x(\varepsilon)
    = \doubleP_{\tau \sim \pi_b(\cdot \mid x)}\bigl[
        T(\tau)\le \eta(x,\varepsilon)\cdot \doubleE_{\tau \sim \pi_b(\cdot \mid x)}[T(\tau)]
      \bigr].
\end{align}
In the deficient relative budget regime where $H \ll \doubleE_{\tau \sim \pi_b(\cdot \mid x)}[T(\tau)]$, any trajectory achieving high reward must solve the problem within the budget $H$. Thus, Lemma~\ref{lemma:anticonc_T_tail} implies that the time threshold
$\eta(x,\varepsilon)\,\doubleE_{\tau \sim \pi_b(\cdot \mid x)}[T(\tau)]$ is comparable to $H$.
More precisely, there exist constants $0<c_1\le c_2<\infty$, independent of $x$,
such that for all sufficiently small $\xi_x$,
\begin{align}
    c_1 H
    \le
    \eta(x,\varepsilon)\,\doubleE_{\tau \sim \pi_b(\cdot \mid x)}[T(\tau)]
    \le
    c_2 H.
\end{align}
Substituting $\xi_x = H / \doubleE_{\tau \sim \pi_b(\cdot \mid x)}[T(\tau)]$, this is equivalent to
\begin{align}
    c_1 \xi_x
    \le 
    \eta(x,\varepsilon)
    \le
    c_2 \xi_x.
\end{align}

We now apply Assumption~\ref{assumption:left_tail_T_tau} with
$t = \eta(x,\varepsilon)\,\doubleE_{\tau \sim \pi_b(\cdot \mid x)}[T(\tau)]$.
Since $\eta(x,\varepsilon) \le c_2 \xi_x$ and we consider $\xi_x \to 0$, for all sufficiently small $\xi_x$ we have $t/\doubleE_{\tau \sim \pi_b(\cdot \mid x)}[T(\tau)] \le z_0$, falling within the validity range of the assumption. Thus,
\begin{align}
    c_-\, f\bigl(\eta(x,\varepsilon)\bigr)
    \le
    c_x(\varepsilon)
    \le
    c_+\, f\bigl(\eta(x,\varepsilon)\bigr).
\end{align}
Since $f$ is non-decreasing and satisfies the doubling property near the origin,
there exist constants $C_1,C_2>0$ such that
\begin{align}
    C_1\, f(\xi_x)
    \le
    f\bigl(\eta(x,\varepsilon)\bigr)
    \le
    C_2\, f(\xi_x).
\end{align}
For all sufficiently small $\xi_x$, we have
\begin{align}
    c_- C_1\, f(\xi_x)
    \le
    c_x(\varepsilon)
    \le
    c_+ C_2\, f(\xi_x),
\end{align}
Therefore, in the regime of $\xi_x\to 0$, the following equation holds:
\begin{align}
    c_x(\varepsilon)
    = \Theta\bigl(f(\xi_x)\bigr)
\end{align}

Finally, fix a small $\xi>0$ and consider the slice
$\mathcal{X}(\xi) \coloneqq \{x : \xi/2 \le \xi_x \le 2\xi\}$.
By the monotonicity and doubling property of $f$, we have
$f(\xi_x) = \Theta(f(\xi))$ uniformly over $x\in\mathcal{X}(\xi)$.
Taking the infimum over $x\in\mathcal{X}(\xi)$ yields
\begin{align}
    c_0(\xi)
    \coloneqq \inf_{x\in\mathcal{X}(\xi)} c_x(\varepsilon)
    = \Theta\bigl(f(\xi)\bigr).
\end{align}
Therefore, we obtained the desired results.
\end{proof}

\subsection{Balanced Relative Budget Regime}
\label{proof_prop:middle_criticality}

In the regime of $\xi_x = \Theta(1)$ (e.g., $H = \Theta(\mu_x)$),
both success and failure occur with probability $\Theta(1)$.
Thus, the truncation at $T(\tau)\le H$ does not substantially change the scale of the moments of $T(\tau)$.
We now provide the following proposition that characterizes the reward variance and anti-concentration in the balanced relative budget regime.

\begin{proposition}[Balanced relative budget]
\label{prop:middle_criticality}
Suppose Assumptions \ref{ass:moments} and \ref{assumption:left_tail_T_tau} hold and set $\varepsilon \le (q_{\min}/2)^2$.
Fix constants $0<\xi_{\min}\le \xi_{\max}<\infty$ and define
$\Xcal_b \coloneqq \{x: \xi_x \in [\xi_{\min}/2,\ 2\xi_{\max}]\}$.
Also, suppose that \cref{ass:balanced-nondeg-min} holds uniformly over $\Xcal_b$ and fix any $\xi \in [\xi_{\min},\xi_{\max}]$.
Then, for all instances $x\in X(\xi)$, the RL reward standard deviation and anti-concentration coefficient satisfy
\begin{align}
    \sigma_{b,x}
    = \Theta(H)
    \quad \text{and} \quad    
    c_0(\xi)
    &= \Theta(1).
    \label{eq:c0_middle_criticality}
\end{align}
\end{proposition}

\begin{proof}

Fix an instance $x \in \Xcal_b$ (i.e., $\xi_x \in [\xi_{\min}, \xi_{\max}]$).
By Assumption~\ref{ass:moments},
\begin{align}
    \doubleV_{\tau \sim \pi_b(\cdot \mid x)}[T(\tau)] = \Theta\left(\doubleE_{\tau \sim \pi_b(\cdot \mid x)}[T(\tau)]^2\right).
\end{align}
In the balanced regime, we have $\mu_x = H/\xi_x = \Theta(H)$. Consequently, the unconditional variance satisfies
\begin{align}
    v_x \coloneqq \doubleV_{\tau \sim \pi_b(\cdot \mid x)}[T(\tau)] = \Theta(H^2).
    \label{eq:proof-prop53-uncond-var}
\end{align}
We now consider the variance conditioned on success. 
By Assumption~\ref{ass:balanced-nondeg-min}, the truncation at $H$ preserves the order of the variance:
\begin{align}
    \tilde{v}_x \coloneqq \doubleV_{\tau \sim \pi_b(\cdot \mid x)}[T(\tau) \mid T(\tau) \le H] 
    \ge c_v v_x.
    \label{eq:proof-prop53-cond-var}
\end{align}
From Lemma~\ref{lemma:reward_moments}, the variance of the reward $R(\tau)$ is lower-bounded by the conditional variance term:
\begin{align}
    \doubleV_{\tau \sim \pi_b(\cdot \mid x)}[R(\tau)] 
    = q_x \tilde{v}_x + q_x(1-q_x)(H+1-\tilde{\mu}_x)^2
    \ge q_x \tilde{v}_x.
\end{align}
Using the bound on the success probability $q_x \ge q_{\min}$ from Assumption~\ref{ass:balanced-nondeg-min} and the variance bound from \eqref{eq:proof-prop53-cond-var}, we obtain
\begin{align}
    \doubleV_{\tau \sim \pi_b(\cdot \mid x)}[R(\tau)] = \Omega(H^2).
\end{align}
Since the reward is bounded by $H+1$, the variance is trivially upper-bounded by $\Ocal(H^2)$. Therefore, the standard deviation scales as $\sigma_{b,x} = \Theta(H)$.

By Lemma~\ref{lemma:anticonc_T_tail}, for each $\varepsilon>0$ there exists $\eta(x,\varepsilon)\in(0,1]$ such that
\begin{align}
c_x(\varepsilon)
=
\doubleP_{\tau \sim \pi_b(\cdot \mid x)}
[T(\tau)\le \eta(x,\varepsilon)\,\doubleE_{\tau \sim \pi_b(\cdot \mid x)}[T(\tau)]].
\end{align}
It suffices to show that $\eta(x,\varepsilon)$ is uniformly bounded away from $0$ over $x\in\Xcal_b$ for a fixed
$\varepsilon$ (equivalently, for $\varepsilon$ in a fixed compact interval).
Recall the definition of $\eta(x,\varepsilon)$
in Lemma~\ref{lemma:anticonc_T_tail}:
\begin{align}
\eta(x,\varepsilon)
=
\frac{\min\{H,\ H+1-\doubleE_{\tau \sim \pi_b(\cdot \mid x)}[R(\tau)]-\sigma_{b,x}\sqrt{\varepsilon}\}}{\doubleE_{\tau \sim \pi_b(\cdot \mid x)}[T(\tau)]}.
\end{align}
Using $R(\tau)\le H$ and $q_x\le 1-q_{\min}$, we have $\doubleE_{\tau \sim \pi_b(\cdot \mid x)}[R(\tau)]\le (1-q_{\min})H$; hence,
\begin{align}
    H+1-\doubleE_{\tau \sim \pi_b(\cdot \mid x)}[R(\tau)]\ge q_{\min}H.
\end{align}
Moreover, $\sigma_{b,x}\le H$ since $R(\tau)\in[0,H]$. Therefore, for any fixed $0<\varepsilon\le (q_{\min}/2)^2$,
\begin{align}
    H+1-\doubleE_{\tau \sim \pi_b(\cdot \mid x)}[R(\tau)] -\sigma_{b,x}\sqrt{\varepsilon}
    \ge (q_{\min}-\sqrt{\varepsilon})H
    \ge \frac{q_{\min}}{2}H.
\end{align}
Since $\doubleE_{\tau \sim \pi_b(\cdot \mid x)}[T(\tau)]=\mu_x=\Theta(H)$ on $\Xcal_b$, we obtain a uniform lower bound
\begin{align}
    \eta(x,\varepsilon)\ \ge\ \frac{(q_{\min}/2)H}{\mu_x}
    =\frac{q_{\min}}{2}\,\xi_x
    \ \ge\ \frac{q_{\min}}{2}\,\xi_{\min}
    =: \eta_{\min}>0.
    \end{align}
Let $\bar\eta(x,\varepsilon):=[\eta(x,\varepsilon)]_0^{z_0}$. 
Then Assumption~\ref{assumption:left_tail_T_tau}
yields
\begin{align}
c_x(\varepsilon)
\ge \doubleP_{\tau \sim \pi_b(\cdot \mid x)}[T(\tau)\le \bar\eta(x,\varepsilon)\,\mu_x]
\ge c_-f(\bar\eta(x,\varepsilon))
\ge c_- f (\min\{\eta_{\min},z_0\})
=: c_\star>0.
\end{align}
Taking the infimum over $x\in X(\xi)$ (with $\xi$ in the balanced range) gives $c_0(\xi)\ge c_\star$, while trivially
$c_0(\xi)\le 1$. Hence $c_0(\xi)=\Theta(1)$ in the balanced regime, proving~\eqref{eq:c0_middle_criticality}.
\end{proof}

\subsection{Ample Relative Budget Regime}
\label{proof_prop:high_criticality}
In the regime of $\xi \to \infty$, we have $H \gg \mu_x$. Failure events with $T(\tau) > H$ are rare, so $q_x \to 1$ and the truncation at $H$ becomes negligible. The reward is approximately $R(\tau) \approx H - T(\tau)$.
The following proposition characterizes how the reward variance and anti-concentration behave in this ample relative budget regime.

\begin{proposition}[Ample relative budget]
\label{prop:high_criticality}
Suppose Assumption~\ref{ass:moments} and \ref{assumption:left_tail_T_tau} hold and
fix $\varepsilon \le \varepsilon_0$ for a sufficiently small universal constant $\varepsilon_0\in(0,1)$.
In the regime $\xi_x \to \infty$, the RL reward standard deviation and anti-concentration coefficient jointly satisfy
\begin{align}
    \sigma_{b,x}
    = \Ocal(H / \xi_x)
    \quad \text{and} \quad
    c_0(\xi)
    = \Theta(1).
    \label{eq:c0_high_criticality}
\end{align}
\end{proposition}

\begin{proof}
In the ample relative budget regime we have $\xi_x \to \infty$.
By Markov's inequality,
\begin{align}
  1-q_x
  =
  \doubleP_{\tau \sim \pi_b(\cdot \mid x)}[T(\tau)>H]
  \le
  \frac{\doubleE_{\tau \sim \pi_b(\cdot \mid x)}[T(\tau)^2]}{H^2}.
\end{align}
By Assumption~\ref{ass:moments}, since $v_x=\doubleV[T(\tau)]=\Theta(\mu_x^2)$, we have
\begin{align}
  \doubleE_{\tau \sim \pi_b(\cdot \mid x)}[T(\tau)^2]
  =
  \doubleV_{\tau \sim \pi_b(\cdot \mid x)}[T(\tau)]+\mu_x^2
  =
  \Theta(\mu_x^2),
\end{align}
and therefore $1-q_x = \Ocal(\mu_x^2/H^2) = \Ocal(1/\xi_x^2)$.

Let $R(\tau)$ be the shaped reward under $\pi_b(\cdot\mid x)$, and recall
$\sigma_{b,x}^2 := \doubleV_{\tau\sim\pi_b(\cdot\mid x)}[R(\tau)]$.
Lemma~\ref{lemma:reward_moments} gives
\begin{align}
  \doubleV_{\tau\sim\pi_b(\cdot\mid x)}[R(\tau)]
  =
  q_x \tilde v_x
  +
  q_x(1-q_x) (H+1-\tilde\mu_x)^2.
  \label{eq:varR-decomp}
\end{align}

Since $\tilde v_x \le \doubleE_{\tau\sim\pi_b(\cdot\mid x)}[T(\tau)^2\mid T(\tau)\le H]$,
\begin{align}
  \tilde v_x
  \le
  \frac{\doubleE_{\tau\sim\pi_b(\cdot\mid x)}[T(\tau)^2 \,\mathbb{I}\{T(\tau)\le H\}]}{\doubleP_{\tau\sim\pi_b(\cdot\mid x)}[T(\tau)\le H]}
  \le
  \frac{\doubleE_{\tau\sim\pi_b(\cdot\mid x)}[T(\tau)^2]}{q_x}
  =
  \Ocal(\mu_x^2),
\end{align}
using $\doubleE_{\tau\sim\pi_b(\cdot\mid x)}[T(\tau)^2]=\Theta(\mu_x^2)$ and $q_x\to 1$.

Next, note that $H=\xi_x\mu_x$ and $\tilde\mu_x\le H$ imply
$0\le H+1-\tilde\mu_x \le H+1 = \Theta(H)$. Hence the second term in
\eqref{eq:varR-decomp} satisfies
\begin{align}
  q_x(1-q_x)\,(H+1-\tilde\mu_x)^2
  \le
  (1-q_x)\,(H+1)^2
  =
  \Ocal \left(\frac{\mu_x^2}{H^2}\right)H^2
  =
  \Ocal (\mu_x^2).
\end{align}
Combining the two bounds in \eqref{eq:varR-decomp} yields
\begin{align}
    \sigma_{b,x}=\sqrt{\doubleV_{\tau\sim\pi_b(\cdot\mid x)}[R(\tau)]} = \Ocal(\mu_x).
\end{align}

By Lemma~\ref{lemma:anticonc_T_tail}, for each $(x,\varepsilon)$ there exists a scalar
$\eta(x,\varepsilon)>0$ such that
\begin{align}
  c_x(\varepsilon)
  =
  \doubleP_{\tau \sim \pi_b(\cdot \mid x)} \left[T(\tau)\le \eta(x,\varepsilon)\,\mu_x\right].
\end{align}
To make the truncation effect explicit, note that
\begin{align}
    \mu_x- \tilde{\mu}_x
    = \frac{\doubleE_{\tau \sim \pi_b(\cdot \mid x)}[T(\tau) \mathbb{I}\{T(\tau)>H\}]}{q_x}
    \le \frac{\sqrt{\doubleE_{\tau \sim \pi_b(\cdot \mid x)}[T(\tau)^2]\doubleP_{\tau \sim \pi_b(\cdot \mid x)}[T(\tau)>H]}}{q_x},
\end{align}
Applying $\doubleP_{\tau \sim \pi_b(\cdot \mid x)}[T(\tau)>H]=O (\mu_x^2/H^2)$ from the Markov inequality and $\doubleE_{\tau \sim \pi_b(\cdot \mid x)}[T(\tau)^2]=\Theta(\mu_x^2)$ by \cref{ass:moments},
we get $\mu_x-\tilde\mu_x=\Ocal(\mu_x^2/H)=\Ocal(\mu_x/\xi_x)$ and hence $\tilde\mu_x/\mu_x\to 1$ as $\xi_x\to\infty$.
Also, Lemma~5.1 gives
\begin{align}
    \eta(x,\varepsilon)\mu_x=\min\{H,\ H+1-\doubleE_{\tau \sim \pi_b(\cdot \mid x)}[R(\tau)]-\sigma_{b,x}\sqrt{\varepsilon}\},
\end{align}
and since $\doubleE_{\tau \sim \pi_b(\cdot \mid x)}[R(\tau)]=q_x(H+1-\tilde\mu_x)$ and $\sigma_{b,x}=\Ocal(\mu_x)$, for all sufficiently large $\xi_x$ we have
$\eta(x,\varepsilon)\ge \tilde\mu_x/\mu_x - C\sqrt{\varepsilon}$ for some constant $C$.
Thus, choosing $\varepsilon_0>0$ small enough yields $\eta(x,\varepsilon)\ge \eta_{\min}>0$ uniformly for all
$\varepsilon\le \varepsilon_0$, which implies $c_x(\varepsilon)=\Theta(1)$ by Assumption~\ref{assumption:left_tail_T_tau}.

Let $z_0$ be the constant from Assumption~\ref{assumption:left_tail_T_tau} and define the clamped value
$\bar\eta(x,\varepsilon):=[\eta(x,\varepsilon)]_0^{z_0}$.
By Assumption~\ref{assumption:left_tail_T_tau}, there exists $z_1\in(0,z_0]$ with $f(z_1)>0$.
Choose $\varepsilon$ small enough so that
$\min\{\eta_{\min},z_0\}\ge z_1$; equivalently, $\bar\eta(x,\varepsilon)\ge z_1$ for all $x$ with $\xi_x\ge\xi_0$.
By monotonicity,
\begin{align}
    c_x(\varepsilon)=\doubleP_{\tau \sim \pi_b(\cdot \mid x)}[T\le \eta(x,\varepsilon)\mu_x]\ \ge\ \doubleP_{\tau \sim \pi_b(\cdot \mid x)}[T\le \bar\eta(x,\varepsilon)\mu_x]
\ \ge\ c_- f(\bar\eta(x,\varepsilon)).
\end{align}
Since $f$ is nondecreasing and satisfies the doubling condition with $f(z_1)>0$, we have $f(z)>0$ for all $z\in(0,z_0]$; hence
$f(\bar\eta(x,\varepsilon))\ge f(\min\{\eta_{\min}, z_0\})=:c_f>0$ and therefore $c_x(\varepsilon)\ge c_-c_f=\Omega(1)$.
Then Assumption~\ref{assumption:left_tail_T_tau} yields, uniformly over such $x$,
\begin{align}
  c_x(\varepsilon)
  &\ge
  \doubleP_{\tau\sim\pi_b(\cdot\mid x)} \left[T(\tau)\le \bar\eta(x,\varepsilon)\,\mu_x \right]
  \ge
  c_-\, f(\bar\eta(x,\varepsilon))
  \ge
  c_-\, f(z_1)
  =: c_\star
  >0.
\end{align}
Hence, for any slice $\Xcal(\xi)$ with $\xi\ge\xi_0$ we have
\begin{align}
  c_0(\xi)=\inf_{x\in\Xcal(\xi)} c_x(\varepsilon)\ \ge\ c_\star,
\end{align}
while trivially $c_0(\xi)\le 1$. Therefore $c_0(\xi)=\Theta(1)$.
\end{proof}

\subsection{Standard deviation under Gamma Model}
\label{proof_lemma:sigma_RL_gamma}

\begin{lemma}[Standard deviation under Gamma model.]
\label{lemma:sigma_RL_gamma}
Let $\gamma$ denote the regularized lower incomplete gamma function.
Define
\begin{align*}
\label{eq:c_rl_gamma}
    \bigl(C_{\mathrm{RL}}(K,\xi)\bigr)^2
    &\coloneqq (K\xi)^2 \gamma(K,K\xi) - 2K^2\xi\,\gamma(K{+}1,K\xi) \\
    &\quad + K(K{+}1)\,\gamma(K{+}2,K\xi) \\
    &\quad - \left(K\xi\,\gamma(K,K\xi) - K\,\gamma(K{+}1,K\xi)\right)^2.
\end{align*}
Then, the standard deviation of the cumulative reward under the gamma model satisfies
\begin{align}
\sigma_b(K,\xi) = \frac{C_{\mathrm{RL}}(K,\xi)}{p}
= \frac{C_{\mathrm{RL}}(K,\xi)}{K\xi}\,H .
\end{align}
\end{lemma}

\begin{proof}
For simplicity, in the Gamma model analysis (Section 5.3 and Appendix A.6), we treat the step count $T$ as a continuous random variable and use the continuous analogue of the bi-level shaped reward $R = \max\{0, H-T\}$
which differs from \cref{eq:R_tau_x} by an additive constant (i.e., ``$+1$'') on successful trajectories.
This simplification does not affect the asymptotic scalings in our regime analyses, and we can obtain the same order-level conclusions.

Define $C_\mathrm{RL}(K,\xi)$ as
\begin{equation}
    \bigl(C_\mathrm{RL}(K,\xi)\bigr)^2 \coloneqq p^2 \doubleV[R(\tau)]
\end{equation}
To find an expression for this quantity, we must first compute the variance of $R(\tau)$, which is given by the standard formula $\doubleV[R(\tau)] = \doubleE[R(\tau)^2] - (\doubleE[R(\tau)])^2$. 
We proceed by calculating the first two raw moments of $R(\tau)$.
For this analysis, the discrete negative binomial distribution of $T(\tau)$ is approximated by a continuous Gamma distribution, $T(\tau) \sim \text{Gamma}(K, p)$, which is a valid approximation for small success probabilities $p$.
The probability density function is
\begin{align}
    f(t; K, p) = \frac{p^K}{\Gamma(K)}t^{K-1}e^{-pt}.
\end{align}

First, we compute the expected reward, $\doubleE[R(\tau)]$.
\begin{align}
    \doubleE[R(\tau)] &= \doubleE[\max(0, H - T(\tau))] \nonumber \\
    &= \int_0^H (H - t) f(t; K,p) \,dt \nonumber \\
    &= H \int_0^H f(t; K,p) \,dt - \int_0^H t f(t; K,p) \,dt
\end{align}
The first integral represents the cumulative distribution function of the Gamma distribution evaluated at $H$, which is the regularized lower incomplete gamma function, $\gamma(K, pH)$.
The second integral corresponds to the first moment of the Gamma distribution truncated at $H$.
Leveraging the results from Lemma A.1, this integral evaluates to $\frac{K}{p}\gamma(K+1, pH)$.
Therefore, the expected reward is:
\begin{equation}
    \doubleE[R(\tau)] = H \cdot \gamma(K, pH) - \frac{K}{p} \gamma(K+1, pH)
\end{equation}

Next, we compute the second raw moment of the reward, $\doubleE[R(\tau)^2]$.
\begin{align}
    \doubleE[R(\tau)^2] &= \doubleE[(\max(0, H - T(\tau)))^2] \nonumber \\
    &= \int_0^H (H - t)^2 f(t; K,p) \,dt \nonumber \\
    &= \int_0^H (H^2 - 2Ht + t^2) f(t; K,p) \,dt \nonumber \\
    &= H^2 \int_0^H f(t) \,dt - 2H \int_0^H t f(t) \,dt + \int_0^H t^2 f(t) \,dt
\end{align}
The first two terms are directly related to the moments calculated for $\doubleE[R(\tau)]$. The third term is the second moment of the truncated Gamma distribution, which evaluates to $\frac{K(K+1)}{p^2}\gamma(K+2, pH)$.
Substituting these results, we get:
\begin{equation}
    \doubleE[R(\tau)^2] = H^2 \gamma(K, pH) - 2H \frac{K}{p} \gamma(K+1, pH) + \frac{K(K+1)}{p^2} \gamma(K+2, pH)
\end{equation}

Now we can compute $\bigl(C_\mathrm{RL}(K,\xi)\bigr)^2 = p^2 \left( \doubleE[R(\tau)^2] - (\doubleE[R(\tau)])^2 \right)$.
\begin{align*}
    \bigl(C_\mathrm{RL}(K,\xi)\bigr)^2 = & \ p^2 \left[ H^2 \gamma(K, pH) - \frac{2HK}{p} \gamma(K+1, pH) + \frac{K(K+1)}{p^2} \gamma(K+2, pH) \right] \\
    & - p^2 \left[ H \gamma(K, pH) - \frac{K}{p} \gamma(K+1, pH) \right]^2 \\
    = & \ [ (pH)^2 \gamma(K, pH) - 2pHK \gamma(K+1, pH) + K(K+1) \gamma(K+2, pH) ] \\
    & - [ pH \gamma(K, pH) - K \gamma(K+1, pH) ]^2
\end{align*}
Finally, we introduce the relative budget  $\xi \coloneqq \frac{pH}{K}$, which implies $pH = K\xi$. Substituting this into the expression above yields the final result as stated in Equation (23) of the paper:
\begin{align}
    \bigl(C_\mathrm{RL}(K,\xi)\bigr)^2 = & \ [ (K\xi)^2 \gamma(K, K\xi) - 2K^2\xi \gamma(K+1, K\xi) + K(K+1) \gamma(K+2, K\xi) ] \nonumber \\
    & - [ K\xi \gamma(K, K\xi) - K \gamma(K+1, K\xi) ]^2
\end{align}
\end{proof}

\cref{fig:c_rl} shows the relationship between $\xi$ and $C_\mathrm{RL}(K, \xi)$ for each $K \in \{1, 2, 5, 10\}$.
Observe that $C_\mathrm{RL}(K, \xi)$ converges to $\sqrt{K}$ when $\xi \to \infty$.

\begin{figure}[h]
    \centering
    \includegraphics[width=0.4\linewidth]{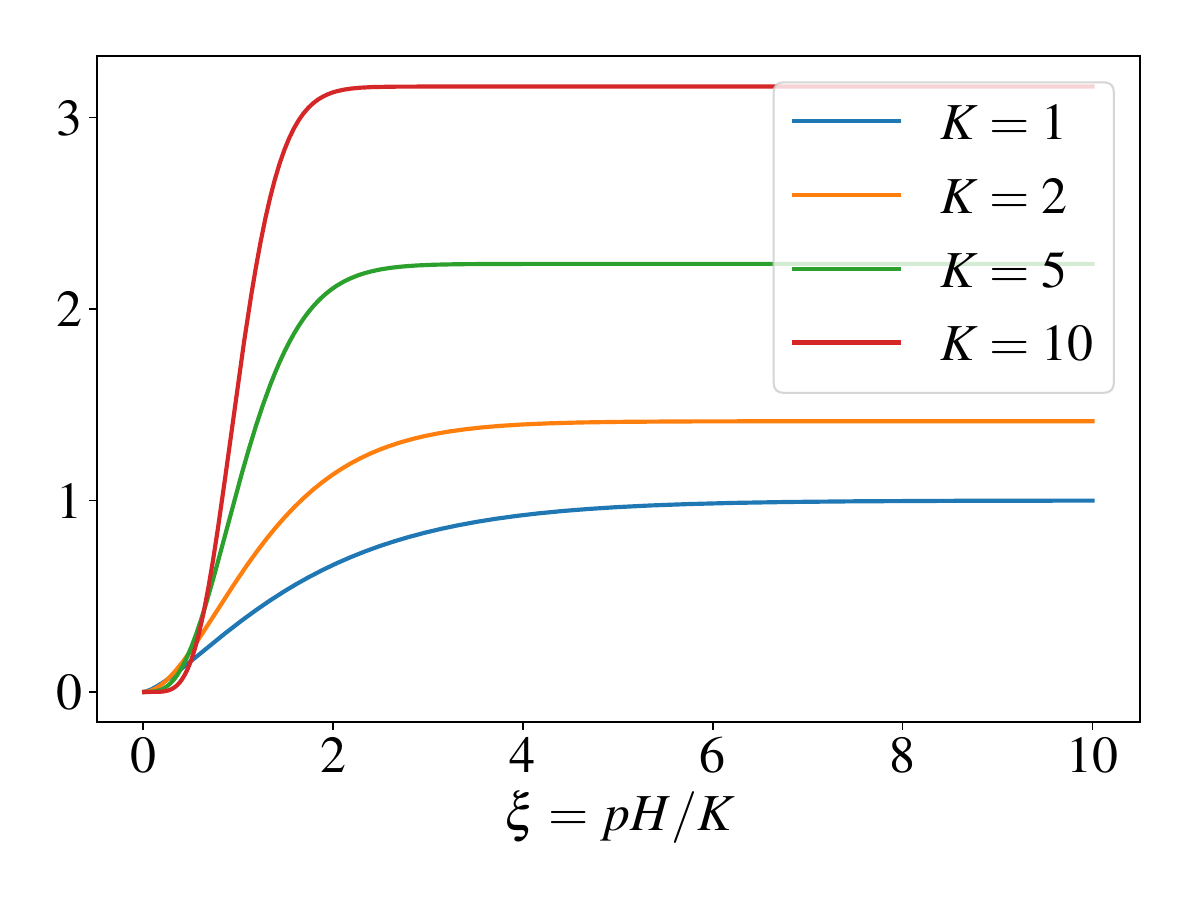}
    \caption{Relationship between $\xi$ and $C_\mathrm{RL}(K, \xi)$.}
    \label{fig:c_rl}
\end{figure}

\subsection{Anti-concentration under Gamma Model}
\label{proof_lemma:anti-concentration-gamma}

\begin{lemma}[Anti-concentration under Gamma model.]
\label{lemma:anti-concentration-gamma}
    Define 
    $\psi (K,\xi,\varepsilon)
    \coloneqq 1 - \gamma(K,K\xi) + \frac{1}{\xi} \gamma(K{+}1,K\xi) - \frac{C_{\mathrm{RL}}(K,\xi)}{K\xi} \sqrt{\varepsilon}$.
    Under the gamma model, the anti-concentration coefficient is represented as
    \begin{align}
    c(K,\xi,\varepsilon) 
    & = \doubleP_{\tau\sim\pi_b(\cdot\mid x)}\left[ T(\tau) \le H \cdot [\psi (K,\xi,\varepsilon)]_0^1 \right] \\
    & = \gamma \bigl(K, K\xi \cdot [\psi(K,\xi,\varepsilon)]_0^1 \bigr).
    \end{align}
\end{lemma}

\begin{proof}
Under the gamma model, we have the following chain of equations for $c(K,\xi,\varepsilon)$:
\begin{align}
\label{eq:c0-gamma}
c(K,\xi,\varepsilon)
&= \doubleP_{\tau\sim\pi_b(\cdot\mid x)}\left[ T(\tau) \le H \left( 1 - \frac{\doubleE[R(\tau)]}{H} - \frac{\sigma_b^{\mathrm{RL}}\sqrt{\varepsilon}}{H}\right) \right] \\
&= \doubleP_{\tau\sim\pi_b(\cdot\mid x)}\left[ T(\tau) \le H \left( 1 - \bigl(\gamma(K,K\xi) - \frac{1}{\xi} \gamma(K{+}1,K\xi)\bigr) - \frac{C_{\mathrm{RL}}(K,\xi)}{K\xi} \sqrt{\varepsilon} \right) \right].
\end{align}
Define $\psi$ such that
\begin{align}
\psi(K,\xi,\varepsilon) &\coloneqq 1 - \Bigl(\gamma(K,K\xi) - \frac{1}{\xi} \gamma(K{+}1,K\xi) \Bigr) - \frac{C_{\mathrm{RL}}(K,\xi)}{K\xi} \sqrt{\varepsilon},
\end{align}
where $[\mathbf{x}]_0^1 \coloneqq \min\{1,\max\{0, \mathbf{x}\}\}$ denotes clipping.
This simplifies to:
\begin{align}
c(K,\xi,\varepsilon) 
& = \doubleP_{\tau\sim\pi_b(\cdot\mid x)}\doubleP\left[ T(\tau) \le H \cdot [\psi (K,\xi,\varepsilon)]_0^1 \right] \\
& = \gamma \bigl(K, K\xi \cdot [\psi(K,\xi,\varepsilon)]_0^1 \bigr).
\end{align}
\end{proof}

\subsection{Optimal Relative Budget}
\label{optimal_opt_xi}

\begin{lemma}[Optimal $\xi^\star$]
\label{lemma:optimal_xi}
For any fixed $K$, the standard deviation
$\sigma_b(K,\xi) = \frac{C_{\mathrm{RL}}(K,\xi)}{K\xi}\,H$
admits a unique global maximizer $\xi^\star(K)\in(0,\infty)$, 
characterized by the first-order condition
\begin{equation}
\xi\,C'_{\mathrm{RL}}(K,\xi) = C_{\mathrm{RL}}(K,\xi).
\label{eq:FOC-gamma}
\end{equation}
\end{lemma}

\begin{proof}
Define
\begin{align}
    g(\xi) = \log\left( \frac{C_{\mathrm{RL}}(K,\xi)}{K\xi}\, H\right) 
    = \log C_{\mathrm{RL}}(K,\xi) + \log H - \log(K\xi).
\end{align}
Differentiating yields
\begin{align}
g'(\xi) = \frac{C'_{\mathrm{RL}}(K,\xi)}{C_{\mathrm{RL}}(K,\xi)} - \frac{1}{\xi}.
\end{align}
The stationary point satisfies
\begin{align}
\xi C'_{\mathrm{RL}}(K,\xi) = C_{\mathrm{RL}}(K,\xi).
\end{align}
Using the convexity of the log-variance function, this stationary point is unique and corresponds to a global maximizer. 
\end{proof}

\section{Proofs from \cref{sec:online_RL}}

\subsection{Proof of \cref{thm:one-step-RL}}
\label{proof_thm:one-step-RL}

\begin{proof}
Due to Lemma A.19 of \citet{setlurscaling}, the best comparator within the $\chi^2$-trust region of radius $\kappa_i$ around the current policy $\pi^{(i)}$ satisfies
\begin{equation}
    J_{\xi_i}(\bar{\pi}_\kappa)
    = J_{\xi_i}\big(\pi^{(i)}\big) + \sqrt{\kappa_i}\,\sigma(\pi^{(i)}).
    \label{eq:chi2-comp}
\end{equation}
Using $n_i$ rollouts from $\pi^{(i)}$, with probability at least $1-\delta$, a verifier-based RL estimator obeys the finite-sample bound
\begin{equation}
    J_{\xi_i}(\bar{\pi}_\kappa) - J_{\xi_i}\big(\hat{\pi}^{(i+1)}\big)
    \le \frac{CH\log(|\mathcal{R}|/\delta)}{c_0(\xi_i; \kappa_i)\,n_i},
    \label{eq:finite-sample}
\end{equation}
where $c_0(\xi_i; \kappa_i)$ is the anti-concentration coefficient with respect to the policy $\pi^{(i)}$ and $C$ is a universal constant implied by \cref{thm:upper}.
Combining \eqref{eq:chi2-comp} and \eqref{eq:finite-sample} gives 
\begin{equation}
    J_{\xi_i}\big(\hat{\pi}^{(i+1)}\big)
    \ge J_{\xi_i}\big(\pi^{(i)}\big)
    + \sqrt{\kappa_i}\,\sigma(\pi^{(i)})
    - \frac{CH\log(|\mathcal{R}|/\delta)}{c_0(\xi_i; \kappa_i)\,n_i}.
\end{equation}
If \eqref{eq:estimator-n} additionally holds, the last term in the above equation is at most
$\tfrac{1}{2}\sqrt{\kappa_i}\,\sigma(\pi^{(i)})$, which implies 
\begin{equation}
    J_{\xi_i} \big(\pi^{(i+1)}\big) = J_{\xi_i} \big(\hat\pi^{(i+1)}\big) \ge J_{\xi_i}\big(\pi^{(i)}\big)
    + \frac{1}{2}\sqrt{\kappa_i}\,\sigma(\pi^{(i)}).
\end{equation}
\end{proof}

\subsection{Proof of \cref{thm:three-regimes_onlineRL}}

Building on the finite-sample guarantee established in \cref{thm:one-step-RL}, we now analyze the specific learning dynamics across the three relative budget regimes introduced in \cref{sec:rl}.
By substituting the regime-specific scaling laws of the standard deviation of rewards and the anti-concentration coefficient, we characterize the minimum sample complexity $n_{i}$ required to ensure monotonic policy improvement. 
The following analyses demonstrate how the relative budget $\xi$ affects the efficiency of online RL, the combination of the following propositions lead to the \cref{thm:three-regimes_onlineRL}.

\begin{proposition}[Deficient relative budget]
    \label{thm:sample_size_low}
    When the relative budget is deficient (i.e., $\xi \to 0$), the attainable per-iteration improvement is limited by $\sqrt{\kappa}\, \sigma(\pi) = \Ocal(H \sqrt{\kappa f(\xi)})$.
    The sufficient sample size is:
    \begin{equation}
        n_i = \widetilde{\Omega}
        \left(
            \frac{1}{\sqrt{\kappa_i}\, (f(\xi_{i}))^{3/2}}
        \right).
    \end{equation}
\end{proposition}
In this regime, the minimum sample size required for policy improvement diverges rapidly as $f(\xi)^{-3/2}$.
This result generalizes the intuition from specific models (e.g., the Gamma distribution where $f(\xi) \sim \xi^K$, yielding $n_i \sim \xi_i^{-3K/2}$).
It highlights the fundamental difficulty of the deficient relative budget regime; that is, as the probability of random success $f(\xi)$ vanishes, the cost of obtaining a sufficient number of informative trajectories grows prohibitively large, regardless of the specific underlying distribution.

\begin{proof}
In the deficient relative budget regime where $\xi_{i} \to 0$, we analyze the behavior using the general left-tail probability function $f(\xi_{i})$ defined in \cref{assumption:left_tail_T_tau}.
According to \cref{prop:low_criticality}, the successful solution traces are rare, and the standard deviation of the reward scales as:
\begin{equation}
    \sigma(\pi^{(i)}) = \Ocal\big(H \sqrt{f(\xi_{i})}\big).
\end{equation}
The minimum number of samples required for policy improvement at iteration $i$ satisfies the general bound:
\begin{equation}
n_i \gtrsim
\frac{H\log(|\mathcal{R}|/\delta)}{c_0(\xi_i; \kappa_i) \, \sqrt{\kappa_i}\,\sigma(\pi^{(i)})}.
\end{equation}
For small $\xi$, \cref{prop:low_criticality} establishes that the anti-concentration coefficient behaves as
\begin{equation}
    c_0(\xi_i; \kappa_i) = \Theta\big(f(\xi_{i})\big).
\end{equation}
Combining this with the asymptotic scaling of $\sigma(\pi^{(i)})$, the sample complexity lower bound becomes:
\begin{equation}
    n_i 
    = 
    \widetilde{\Omega}
    \left(\frac{1}{\sqrt{\kappa_i}\, (f(\xi_{i}))^{3/2}}
    \right).
\end{equation}
\end{proof}

\begin{proposition}[Balanced relative budget]
    \label{thm:sample_size_middle}
    In the balanced relative budget regime where $\xi = \Theta(1)$, the attainable per-iteration improvement is $\Ocal(\sqrt{\kappa} H)$.
    In addition, the sample complexity is:
    \begin{equation}
        n_i = \widetilde{\Omega}(1/\sqrt{\kappa_i}).
    \end{equation}
\end{proposition}
This result highlights that the balanced relative budget regime represents the optimal sweet spot for learning, in the sense of sample complexity.
In this regime, the reward variance is maximized, providing a strong learning signal, while the anti-concentration coefficient remains stable, ensuring that high-reward trajectories are sampled frequently. Consequently, the sample complexity remains constant regardless of the horizon, allowing for highly efficient policy improvement.

\begin{proof}
We now analyze the middle stage, defined by the regime where the relative budget is $\xi_{i} = \Theta(1)$. In this regime, the compute budget $H$ is comparable to the expected time-to-solution $\doubleE[T(\tau)]$, meaning that both success and failure events occur with probability $\Theta(1)$.

According to \cref{prop:middle_criticality}, in the balanced relative budget regime, the truncation at $H$ does not alter the scale of the moments. Consequently, the reward standard deviation and the anti-concentration coefficient satisfy:
\begin{equation}
    \sigma(\pi^{(i)}) = \Theta(H)
    \quad \text{and} \quad
    c_0(\xi_i; \kappa_i) = \Theta(1).
\end{equation}
Using these scaling laws, we analyze the learning dynamics.
The attainable per-iteration improvement is proportional to $\sqrt{\kappa_i}\,\sigma(\pi^{(i)})$.
Because the standard deviation of rewards is maximized in this regime (scaling linearly with $H$ rather than decaying as in the early or late stages), we anticipate the largest magnitude of policy improvement.
The minimum number of samples $n_i$ required for policy improvement is determined by the ratio of the budget to the effective learning signal. Substituting the scalings from \cref{prop:middle_criticality} yields:
\begin{equation}
    n_i = \widetilde{\Omega}(1/\sqrt{\kappa_i}).
\end{equation}
\end{proof}

\begin{proposition}[Ample relative budget]
    \label{thm:sample_size_high}
    In the ample relative budget regime where $\xi \to \infty$, the attainable per-iteration improvement decays linearly in $\xi$; that is, $\sqrt{\kappa} H/\xi$.
    Also, the sample complexity lower bound becomes:
    \begin{equation}
        n_i = \widetilde{\Omega}\left(\xi_{i}/\sqrt{\kappa_i}\right).
    \end{equation}
\end{proposition}
Crucially, ensuring monotonic improvement typically requires scaling the trust region as $\sqrt{\kappa_i} = \Theta(\xi_i^{-1})$.
This results in an effective improvement rate of $\Ocal(H/\xi_i^2)$ and a quadratic degradation in sample complexity: $n_i = \widetilde{\Omega}(\xi_{i}^2)$.

\begin{proof}
We then analyze the regime $\xi_{i} \to \infty$. According to \cref{prop:high_criticality}, in this ample relative budget regime, the success probability approaches 1 ($q_x \to 1$), and the failure events become negligible. Consequently, the truncation effects vanish, and the reward statistics converge to those of the unconditioned distribution.

Under the general moment assumptions (\cref{ass:moments}), the standard deviation and the performance gap scale as follows:
\begin{align}
    \sigma(\pi^{(i)})
    = \Ocal\left(H/\xi_{i}\right).
\end{align}
Thus, both the variance and the sub-optimality gap decay as $\Ocal(H/\xi_{i})$.

Regarding the anti-concentration coefficient, \cref{prop:high_criticality} states that in this regime, $c_0(\xi_i; \kappa_i)$ does not vanish but remains bounded away from zero:
\begin{equation}
    c_0(\xi_i; \kappa_i) = \Theta(1).
\end{equation}
This contrasts with the specific Gamma case where vanishing behavior was possible under certain parameter choices; the general theory guarantees a persistent learning signal provided the trust region is properly managed.

The minimum number of samples required for policy improvement at iteration $i$ is governed by the general bound:
\begin{equation}
    n_i \gtrsim
    \frac{H\log(|\mathcal R|/\delta)}
    {c_0(\xi_i; \kappa_i)\,\sqrt{\kappa_i}\,\sigma(\pi^{(i)})}.
\end{equation}
Substituting the general asymptotics $\sigma(\pi^{(i)}) = \Ocal(H/\xi_{i})$ and $c_0(\xi_i; \kappa_i) = \Theta(1)$, we obtain:
\begin{align}
    n_i 
    = \Omega\left( 
    \frac{H\log(|\mathcal R|/\delta)}{1 \cdot \sqrt{\kappa_i}\,\big(H/\xi_{i}\big)} \right)
    = \Omega\left( \frac{\xi_{i}}{\sqrt{\kappa_i}}\log(|\mathcal R|/\delta) \right).
\end{align}
Thus, the sample size requirement grows linearly with $\xi$.
Meanwhile, the attainable improvement width shrinks as $\Theta(H/\xi)$, showing that in the late stage, progress becomes increasingly expensive in terms of sample complexity, a conclusion that holds broadly beyond specific probabilistic models.
\end{proof}

\subsection{Proof of \cref{thm:online_rl}}
\label{proof_thm:online_rl}

Recall that, for each $i \in \{0,1,\ldots,\}$,
\begin{align}
  c_0^{(i)} \coloneqq \mathbb{P}_{x\sim\rho,\tau\sim\pi^{(i)}(\cdot\mid x)}\left[R(\tau) \ge J(\pi^{(i)}) + \sqrt{\kappa_i}\, \sigma(\pi^{(i)})\right]. 
  \label{eq:def-anticon-gamma}
\end{align}
denotes the anti-concentration coefficient of the policy $\pi^{(i)}$, with the trust region radius
\begin{align}
    \kappa_i = \left(\min\left\{\frac{\sigma(\pi^{(i)})}{J(\pi^{(i)})}, \frac{H - J(\pi^{(i)})}{2\,\sigma(\pi^{(i)})}\right\}\right)^2, \label{eq:def-kappa-i}
\end{align}
which ensures under the gamma model that $c_0^{(i)} > 0$.
Due to a similar argument to \cref{thm:one-step-RL} and by the union bound, it holds with probability at least $1 - \delta$ that
\begin{align}
  & J(\pi^{(i+1)}) \ge J(\pi^{(i)}) + \frac{1}{2} \sqrt{\kappa_i}\, \sigma(\pi^{(i)})
  & \forall i \in \{0,1,\ldots,m-1\},
  \label{eq:rl-online-incr:per-step}
\end{align}
when using sufficiently many rollout samples from each $\pi^{(i)}$ such that
\begin{align}
  n_i 
  \ge \frac{2\,C H \log \left(\delta^{-1}m\left|\Rcal\right|\right)}{\sqrt{\kappa_i}\,c^{(i)}_0 \sigma(\pi^{(i)})}.
\end{align}
Under the gamma model, the moment calculation in \cref{proof_lemma:sigma_RL_gamma} gives that
\begin{align}
  J(\pi^{(i)})
  & = H \left(\gamma\left(K, K \xi_{i}\right) - \frac{1}{\xi_{i}} \gamma\left(K+1, K \xi_{i}\right)\right), \\
  \begin{split}    
    \sigma^2(\pi^{(i)})
    & = H^{2} \left(\gamma\left(K, K \xi_{i}\right)
    - \frac{2}{\xi_{i}}\,\gamma\left(K+1, K \xi_{i}\right)
    + \frac{1 + K^{-1}}{\xi^2_{i}}\,\gamma\left(K+2, K \xi_{i}\right) \right. \\
    & \qquad\qquad \left. - \left(\gamma\left(K, K \xi_{i}\right) - \frac{1}{\xi_{i}}\,\gamma\left(K+1, K \xi_{i}\right)\right)^{2}\right),
  \end{split}
\end{align}
where $\gamma\left(s,\lambda\right)$ denotes the regularized lower incomplete gamma function:
\begin{align}
  & \gamma\left(s,\lambda\right) \coloneqq \frac{\int_{0}^{\lambda} x^{s-1} e^{-x} \,\mathrm{d}x}{\int_{0}^{\infty} x^{s-1} e^{-x} \,\mathrm{d}x}
  & \forall s,\lambda \in \mathbb{R}_{>0}.
\end{align}
Given \cref{eq:def-anticon-gamma,eq:def-kappa-i} and the fact that $0 < J(\pi^{(0)}) < J(\pi^{(1)}) < \cdots \le H$, we see that $\lim_{i\to\infty} \sigma^2(\pi^{(i)}) = 0$, $\lim_{i\to\infty} J(\pi^{(i)}) = H$, and hence that $\lim_{i\to\infty} \xi_{i} = \infty$.
In the limit, we have the following asymptotic expansions:
\begin{align}
  J(\pi^{(i)})
  & = H \left(1 - \frac{1}{\xi_{i}} + O_K\!\left(\xi_{i}^{K-1} e^{-K\xi_{i}}\right)\right); \label{eq:J-asymp} \\
  \begin{split}
    \sigma^2(\pi^{(i)})
    & = H^2 \left(\frac{1}{\xi^2_{i} K} + O_K\!\left(\xi_{i}^{K-1} e^{-K\xi_{i}}\right)\right). 
  \end{split} \label{eq:sigma^2-asymp}
\end{align}
There also exists a constant $\tau(K) \in \{0,1,\ldots\}$ that depends on $K$ and satisfies that $\xi_{\tau(K)} \ge 2\left(1 + K^{-1}\right)$.
When $i\ge \tau(K)$, it holds that
\begin{align}
   \frac{H - J(\pi^{(i)})}{2\,\sigma(\pi^{(i)})} - \frac{\sigma(\pi^{(i)})}{J(\pi^{(i)})}
  = \frac{1}{2 \xi_{i} \sigma(\pi^{(i)}) J(\pi^{(i)})} \left(\gamma\left(K+1, K\xi_{i}\right) - \frac{2\left(1 + K^{-1}\right)}{\xi_{i}} \gamma\left(K+2, K\xi_{i}\right)\right)
  > 0,
\end{align}
implying that $\kappa_{i} = \left(\frac{\sigma(\pi^{(i)})}{J(\pi^{(i)})}\right)^2$.
Plugging these into \cref{eq:rl-online-incr:per-step} gives for each $i \in \{\tau(K),\tau(K)+1,\ldots,m-1\}$ that
\begin{align}
  1 - \frac{1}{\xi_{i+1}} + O_K\!\left(\xi_{i+1}^{K-1} e^{-K\xi_{i+1}}\right)
  & \ge 1 - \frac{1}{\xi_{i}}
    + \frac{1}{2K} \frac{1}{\xi^2_{i}}
    + O_K\!\left(\xi_{i}^{K-1} e^{-K\xi_{i}}\right),
\intertext{which implies, as $\xi_{i} < \xi_{i+1}$, that}
  \frac{1}{\xi_{i+1}}
  & \le \frac{1}{\xi_{i}} - \frac{1}{2K} \frac{1}{\xi^2_{i}} + O_K\!\left(\xi_{i}^{K-1} e^{-K\xi_{i}}\right),
\end{align}
and hence that
\begin{align}
  \xi_{i+1}
    & \ge \xi_{i}\left(1 - \frac{1}{2K} \frac{1}{\xi_{i}} + O_K\!\left(\xi_{i}^{K} e^{-K\xi_{i}}\right)\right)^{-1}
    = \xi_{i} + \frac{1}{2K} - O_K\!\left(\xi_{i}^{K+1} e^{-K\xi_{i}}\right).
\end{align}
Because $\lim_{i\to\infty} \xi_{i} = \infty$, it follows for each $i \in \{0,1,\ldots,m-1\}$ that
\begin{align}
  \xi_{i} - \xi_0 \ge \frac{i}{2K} - O_{K}(1).
\end{align}
Finally, it suffices to show that $c_0^{(i)} = \Omega_K\!\left(1\right)$.
Using \cref{eq:J-asymp,eq:sigma^2-asymp} again, we see that
\begin{align}
  H - J(\pi^{(i)}) - \frac{\sigma^2(\pi^{(i)})}{J(\pi^{(i)})}
  = \frac{H}{\xi_{i}} \left(1 - O_K\!\left(\frac{1}{\xi_{i}}\right)\right).
\end{align}
Together with the CDF of the gamma function, it holds when $i \ge \tau(K)$ that
\begin{align}
  c_0^{(i)}
  & = \gamma\left(K, \frac{\xi_{i} K}{H}\left(H - J(\pi^{(i)}) - \frac{\sigma^2(\pi^{(i)})}{J(\pi^{(i)})}\right)\right)
  = \gamma\left(K, K - O_K\!\left(\frac{1}{\xi_{i}}\right)\right),
\end{align}
which completes the proof of \cref{thm:online_rl}.

\section{Theoretical Analysis on SFT}
\label{appendix:sft}

VF methods (e.g., SFT) rely on mimicking expert traces.
Their performance is fundamentally limited by the heterogeneity of the solution space.
Intuitively, if there are many valid paths of varying lengths and styles to a correct answer, the model struggles to resolve the ambiguity of which path to mimic with limited data.
We quantify this using \textit{policy heterogeneity}, defined as the accumulated variance of the Q-values along the trajectory.
$Q_\pi(s_h, a_h)$ denotes the expected cumulative reward attained by a given LLM $\pi$, in expectation over test problems, defined as
\begin{align}
    Q_\pi(s_h, a_h) \coloneqq \doubleE_{\rho, \pi}\left[\sum_{t=h}^H r(s_t, a_t) \mid s_h, a_h\right].
\end{align}

\begin{definition}[Policy heterogeneity]\label{prop:heterogeneity}
For a problem $x \in \Xcal$ and policy $\pi \in \Pi$, the heterogeneity is defined as:
\begin{equation}
\sigma_{\pi,x}^{2}
  \coloneqq \sum_{h=1}^{H}
     \doubleE_{s_h\sim d^\pi_h}
         \Bigl[\doubleV_{a\sim\pi(\cdot\mid s_h)}
         \bigl[Q_{\pi_e}(s_h,a_h)\bigr]\Bigr].
\end{equation}
Aggregating over the input distribution \(\rho\) gives
\begin{equation}
  \sigma_{\pi}^{2}
  \coloneqq \doubleE_{x\sim\rho}[\sigma_{\pi,x}^{2}], \quad
  \tilde\sigma_{\pi}
  \coloneqq \operatorname{Median}_{x}(\sigma_{\pi,x}).
\end{equation}
\end{definition}

Based on the above definition, we present the sub-optimality gap of SFT that scales as follows:

\begin{lemma}[Information-theoretic lower bound on SFT, Theorem 5.4 in \citet{setlurscaling}]
\label{lemma:lower_sft}
Fix a distribution $\rho$, a reward class $\mathcal R$,
base policy $\pi_{b}$, expert $\pi_{e}$,
and any integer $\nu \le |\mathcal X|/4$.
There exists an expert family $\Pi'$ with $|\Pi'|=2^{\nu}$ and a reward
subclass $\mathcal R'\subset\mathcal R$ such that every VF learner
$\hat\pi^{\mathrm{SFT}}_{\bar{n}}$ trained on $\bar{n}$ successful demonstrations must satisfy
\begin{align*}
    \max_{\pi'\in\Pi'}
    \max_{r'\in\mathcal R'}
    \bigl[J_{r'}(\pi')-J_{r'}(\hat\pi^{\mathrm{SFT}}_{\bar{n}})\bigr]
    =
    \Omega \left(
        \sigma_b \, \sqrt{\frac{\log|\Pi'|}{\bar{n}}}
        \right),
\end{align*}
where $\sigma_b \coloneqq \sigma_{\pi_b}$ is the heterogeneity
parameter of the base policy defined in \cref{prop:heterogeneity}.
\end{lemma}

While the rigorous lower bound depends on the expert's specific heterogeneity $\tilde{\sigma}_e$, we present the bound in terms of the base policy's heterogeneity $\sigma_b$.
This substitution is valid because the expert policy is constrained to lie within a $\chi^2$-neighborhood of the base policy, implying that $\sigma_e \approx \sigma_b$ (Lemma 5.3 in \citet{setlurscaling}).

\subsection{Characterizing Sub-optimality Gap of SFT with Relative Budget}

We characterize the sub-optimality of SFT with the relative budget $\xi$.
Since the SFT dataset contains only successful solution traces, this notion of heterogeneity reduces to the variance of the reward distribution conditioned on success.
We now define this heterogeneity parameter for each problem $x$.
By \eqref{eq:R_tau_x}, the variance for a given $x$ is:
\begin{equation}
    \doubleV[R(\tau) \mid T(\tau) \le H, x] = \doubleV[T(\tau) \mid T(\tau) \le H, x].
\end{equation}

We now introduce a scaling factor $C_{\mathrm{SFT}}(\xi_x) \in [0,1]$ so as to relate the truncated and non-truncated variances for problem $x$:
\begin{equation}
    \doubleV[T(\tau) \mid T(\tau) \le H, x] 
    = \doubleV[T(\tau) \mid x] \cdot \bigl(C_\mathrm{SFT}(\xi_x)\bigr)^2.
\end{equation}
where $\xi_x$ is the problem-dependent criticality from \Cref{definition:criticality_x}.
This scaling factor $C_\mathrm{SFT}(\xi_x)$ has the following qualitative properties.
First, as $\xi_x \to 0$ (i.e., $H \to 0$ for a fixed $\mu_x$), the variance of successful traces $\doubleV[T \mid T \le H, x]$ approaches $0$; thus, $\lim_{\xi_x \to 0} C_\mathrm{SFT}(\xi_x) = 0$.
Also, as $\xi_x \to \infty$ (i.e., $H \to \infty$), the truncation condition $T \le H$ becomes vacuous for problem $x$. 
From \Cref{ass:moments}, $P(T < \infty \mid x) = 1$.
The truncated variance approaches the full variance $\doubleV[T(\tau) \mid x]$; hence, $\lim_{\xi_x \to \infty} C_\mathrm{SFT}(\xi_x) = 1$.

Consequently, the standard deviation for problem $x$ scales as
\begin{align}
    \sigma_{b,x}^{\mathrm{SFT}}(\xi_x)
    \coloneqq
    \sqrt{\doubleV \bigl[T(\tau)\mid T(\tau)\le H,\,x\bigr]}
    =
    \Theta\bigl(\mu_x \cdot C_{\mathrm{SFT}}(\xi_x)\bigr)
    =
    \Theta\left(\frac{C_{\mathrm{SFT}}(\xi_x)}{\xi_x}\,H\right),
\end{align}
where we used $\xi_x = H/\mu_x$ from \Cref{definition:criticality_x}. 
Following \citet{setlurscaling}, the aggregate SFT sub-optimality will depend on a robust aggregate of these per-problem heterogeneities; we use the median $\tilde{\sigma}_b^{\mathrm{SFT}} \coloneqq \mathrm{Median}_{x}\{\sigma_{b,x}^{\mathrm{SFT}}(\xi_x)\}$.

\begin{theorem}[Sub-optimality gap of SFT with problem-dependent criticality]
Let $\bar{n}$ denote the total number of successful solution traces collected.
\begin{align}
    J_{r}(\bar\pi_{\kappa}) - J_{r}(\hat\pi^{\mathrm{SFT}}_{\bar{n}})
    &= \Omega\left(\frac{\tilde{\sigma}_b^\mathrm{SFT}}{\sqrt{\bar{n}}}\right),
    \label{eq:theorem_sft_1_x}
\end{align}
where \(\tilde{\sigma}_b^{\mathrm{SFT}} = \mathrm{Median}_{x}\bigl\{\Theta\bigl(H\,C_{\mathrm{SFT}}(\xi_x)/\xi_x\bigr)\bigr\}\).
\end{theorem}

The gap still scales as $1/\sqrt{\bar n}$, implying slow improvement with more expert data.
For a fixed problem $x$, the factor $C_{\mathrm{SFT}}(\xi_x)/\xi_x$ typically attains its maximum at an intermediate $\xi_x=\Theta(1)$, highlighting a regime where imitation is particularly difficult due to high heterogeneity among successful traces.
The aggregate (median) gap is dominated by problems in this intermediate regime.

\subsection{Concrete Analysis with Gamma Distribution}

As with the analyses on RL, we also analyze a specific setting in which the criticality does not vary across instances and the time to correct solution can be modeled via a Gamma distribution.
The variance of a truncated negative binomial distribution does not admit a simple closed-form expression. In the small-$p$ regime, however, the negative binomial distribution is well-approximated by a Gamma distribution, which allows us to obtain tractable closed-form expressions for its truncated moments.
Here, the variance can be related to the non-truncated variance ($\doubleV[T(\tau)] \approx K/p^2$) via a scaling factor $\bigl(C_\mathrm{SFT}(K, \xi)\bigr)^2$ defined as follows:
\begin{lemma}
    \label{lemma:c_sft}
    Let $\Gamma(\mathbf{x}) \coloneq \int_0^\infty y^{\mathbf{x}-1} e^{-y} dy$ be a gamma function.
    Using the regularized lower incomplete gamma function $\gamma(\mathbf{x}_1, \mathbf{x}_2) \coloneqq \frac{1}{\Gamma(\mathbf{x}_1)} \int_0^{\mathbf{x}_2} y^{\mathbf{x}_1 - 1} e^{-y} dy$, let us define the scaling factor $C_\mathrm{SFT}: \doubleZ_+ \times \doubleR_+ \to [0, 1]$ as
    \begin{equation}
    \label{eq:c_sft_gamma}
    \bigl(C_\mathrm{SFT}(K, \xi)\bigr)^2 = (K+1) \frac{\gamma(K+2, K\xi)}{\gamma(K, K\xi)} - K \left(\frac{\gamma(K+1, K\xi)}{\gamma(K, K\xi)}\right)^2.
    \end{equation}
    Then, the following equation holds:
    \begin{equation}
    \doubleV[T(\tau) \mid T(\tau) \le H] = \doubleV[T(\tau)] \cdot \bigl(C_\mathrm{SFT}(K, \xi)\bigr)^2 \approx \frac{K}{p^2} \cdot \bigl(C_\mathrm{SFT}(K, \xi)\bigr)^2.
\end{equation}
\end{lemma}

\begin{proof}
Let $X$ be a random variable following a Gamma distribution with shape $K$ and rate $p$, denoted $X \sim \text{Gamma}(K, p)$.
We aim to find its variance, conditional on the event $X \le H$, where $H$ is the token budget.
First of all, the variance of any random variable $X$ truncated from above at $H$ is given by the standard formula:
\begin{equation}
    \doubleV[X | X \le H] = \doubleE[X^2 | X \le H] - \left(\doubleE[X | X \le H]\right)^2
\end{equation}
Here, our task reduces to finding the first and second moments of the truncated Gamma distribution.

Let $X \sim \text{Gamma}(K, p)$.
The $k$-th moment of $X$ truncated at $H$ is defined as:
\begin{equation}
    \doubleE[X^k | X \le H] = \frac{\int_0^H x^k f(x) \,dx}{\int_0^H f(x) \,dx}
\end{equation}
where $f(x)$ is the probability density function of the Gamma distribution:
\begin{equation}
    f(x; K, p) = \frac{p^K}{\Gamma(K)} x^{K-1} e^{-px}
\end{equation}
The denominator, $\int_0^H f(x) \,dx$, is the cumulative distribution function (CDF) evaluated at $H$. For the Gamma distribution, this is the regularized lower incomplete gamma function, $\gamma(K, pH)$.
We compute the integral in the numerator for an arbitrary moment $k$:
\begin{align*}
\int_0^H x^k f(x) \,dx &= \int_0^H x^k \frac{p^K}{\Gamma(K)} x^{K-1} e^{-px} \,dx \\
&= \frac{p^K}{\Gamma(K)} \int_0^H x^{K+k-1} e^{-px} \,dx
\end{align*}
After a change of variables ($t=px$), the integral can be expressed using the regularized lower incomplete gamma function $\gamma(s,z)$. Since the integral $\int_0^z t^{s-1} e^{-t} \,dt$ is equal to $\Gamma(s)\gamma(s,z)$ by definition, we have:
\begin{align*}
\frac{p^K}{\Gamma(K)} \int_0^H x^{K+k-1} e^{-px} \,dx
&= \frac{p^K}{\Gamma(K)} \int_0^{pH} \left(\frac{t}{p}\right)^{K+k-1} e^{-t} \frac{dt}{p} \\
&= \frac{1}{p^k \Gamma(K)} \int_0^{pH} t^{K+k-1} e^{-t} \,dt \\
&= \frac{1}{p^k \Gamma(K)} \left[ \Gamma(K+k) \gamma(K+k, pH) \right] \\
&= \frac{\Gamma(K+k)}{p^k \Gamma(K)} \gamma(K+k, pH)
\end{align*}

The first moment of $X$ truncated from above at $H$ is given by:
\begin{equation}
    \doubleE[X | X \le H] = \frac{\frac{\Gamma(K+1) \gamma(K+1, pH)}{p \Gamma(K)}}{\gamma(K, pH)} = \frac{K}{p} \frac{\gamma(K+1, pH)}{\gamma(K, pH)}.
\end{equation}
Also, the second moment of $X$ truncated from above at $H$ is given by:
\begin{equation}
    \doubleE[X^2 | X \le H] = \frac{\frac{\Gamma(K+2) \gamma(K+2, pH)}{p^2 \Gamma(K)}}{\gamma(K, pH)} = \frac{K(K+1)}{p^2} \frac{\gamma(K+2, pH)}{\gamma(K, pH)}.
\end{equation}

Substitute the moments back into the variance formula:
\begin{equation}
    \doubleV[X | X \le H] = \frac{K(K+1)}{p^2} \frac{\gamma(K+2, pH)}{\gamma(K, pH)} - \left( \frac{K}{p} \frac{\gamma(K+1, pH)}{\gamma(K, pH)} \right)^2.
\end{equation}

The scaling factor is the ratio of the truncated variance to the non-truncated variance, where $\doubleV[X] = K/p^2$.
\begin{align*}
    \bigl(C_\mathrm{SFT}(K, \xi)\bigr)^2 &= \frac{\doubleV[X | X \le H]}{\doubleV[X]} \\
    &= \frac{\frac{K(K+1)}{p^2} \frac{\gamma(K+2, pH)}{\gamma(K, pH)} - \frac{K^2}{p^2} \left( \frac{\gamma(K+1, pH)}{\gamma(K, pH)} \right)^2}{K/p^2} \\
    &= (K+1)\frac{\gamma(K+2, pH)}{\gamma(K, pH)} - K\left(\frac{\gamma(K+1, pH)}{\gamma(K, pH)}\right)^2.
\end{align*}
Finally, substituting the criticality parameter $\xi = pH/K$ gives
\begin{equation}
    \bigl(C_\mathrm{SFT}(K, \xi)\bigr)^2 = (K+1)\frac{\gamma(K+2, K\xi)}{\gamma(K, K\xi)} - K\left(\frac{\gamma(K+1, K\xi)}{\gamma(K, K\xi)}\right)^2.
\end{equation}
\end{proof}
This allows us to write the standard deviation $\sigma_b^\text{SFT}$ as:
\begin{align}
\sigma_b^\text{SFT}(K, \xi) = \sqrt{\doubleV[T(\tau) \mid T(\tau) \le H]} \approx \frac{\sqrt{K}}{p} C_\mathrm{SFT}(K, \xi).
\end{align}
Using the criticality parameter $\xi = pH/K$, this can also be expressed as:
\begin{align}
    \sigma_b^\text{SFT}(K, \xi) \approx \frac{C_\mathrm{SFT}(K, \xi)}{\sqrt{K}\xi} H,
\end{align}
where $C_\mathrm{SFT}(K, \xi) \in \doubleR_+$ is a positive scalar in $[0, 1]$ that depends on both $K$ and $\xi$.
As shown in \Cref{fig:sft_1}, for any fixed $K$, it behaves similarly to the geometric case ($K=1$); that is, $\lim_{\xi \to 0} C_\mathrm{SFT}(K, \xi) = 0$ and $\lim_{\xi \to \infty} C_\mathrm{SFT}(K, \xi) = 1$. 
This function captures the reduction in variance due to truncation.

\begin{theorem}[Sub-optimality gap of SFT]
Let $\bar{n}$ denote the number of the successful solution traces.
\begin{align}
    J_{r}(\bar\pi_{\kappa}) - J_{r}(\hat\pi^{\mathrm{SFT}}_{\bar{n}})
    &= \Omega\left(\frac{\sigma_b^\mathrm{SFT}}{\sqrt{\bar{n}}}\right)
    = \Omega\left(C_\mathrm{SFT}(K, \xi) \cdot \frac{\sqrt{K}}{p \sqrt{\bar{n}}}\right) \\
    &= \Omega\left(\frac{C_\mathrm{SFT}(K, \xi)}{\sqrt{K}\xi} \cdot \frac{H}{\sqrt{\bar{n}}}\right).
    \label{eq:theorem_sft_1}
\end{align}
\end{theorem}

The sub-optimality gap still scales with $1/\sqrt{\bar{n}}$, indicating slow improvement with more expert data.
Also, SFT's performance degrades as the average problem difficulty $K/p = \doubleE[T(\tau)]$ increases.
Furthermore, for a fixed problem difficulty (fixed $K, p$), increasing the compute budget $H$ (and thus $\xi$) increases the final term $\frac{C_\mathrm{SFT}(K, \xi)}{\sqrt{K}\xi} H$, as $C_\mathrm{SFT}$ approaches 1 while $\xi$ grows linearly with $H$.
The term $\frac{C_\mathrm{SFT}(K,\xi)}{\sqrt{K}\xi}$, which dictates the scale of the sub-optimality, is plotted as a function of $\xi$ in \Cref{fig:sft_2}.
The plot reveals a critical insight: this term reaches a maximum value around $\xi = 1$, indicating that the sub-optimality gap is at its largest under this condition.
This peak highlights a key limitation of SFT: its performance deteriorates when the compute budget $H$ is on the same order as the expected time-to-solution $\doubleE[T(\tau)]$.
This regime maximizes the variance among successful traces, making imitation particularly difficult.
This is because such a condition maximizes the heterogeneity of the successful solution traces used for training. The high variance in these successful examples, which range from very fast solutions to those that barely finish within the budget, makes it exceptionally difficult for the model to learn a single, coherent strategy through imitation.

\begin{figure}[t]
  \centering
  \begin{subfigure}[b]{0.48\textwidth}
    \centering
    \includegraphics[height=4.5cm]{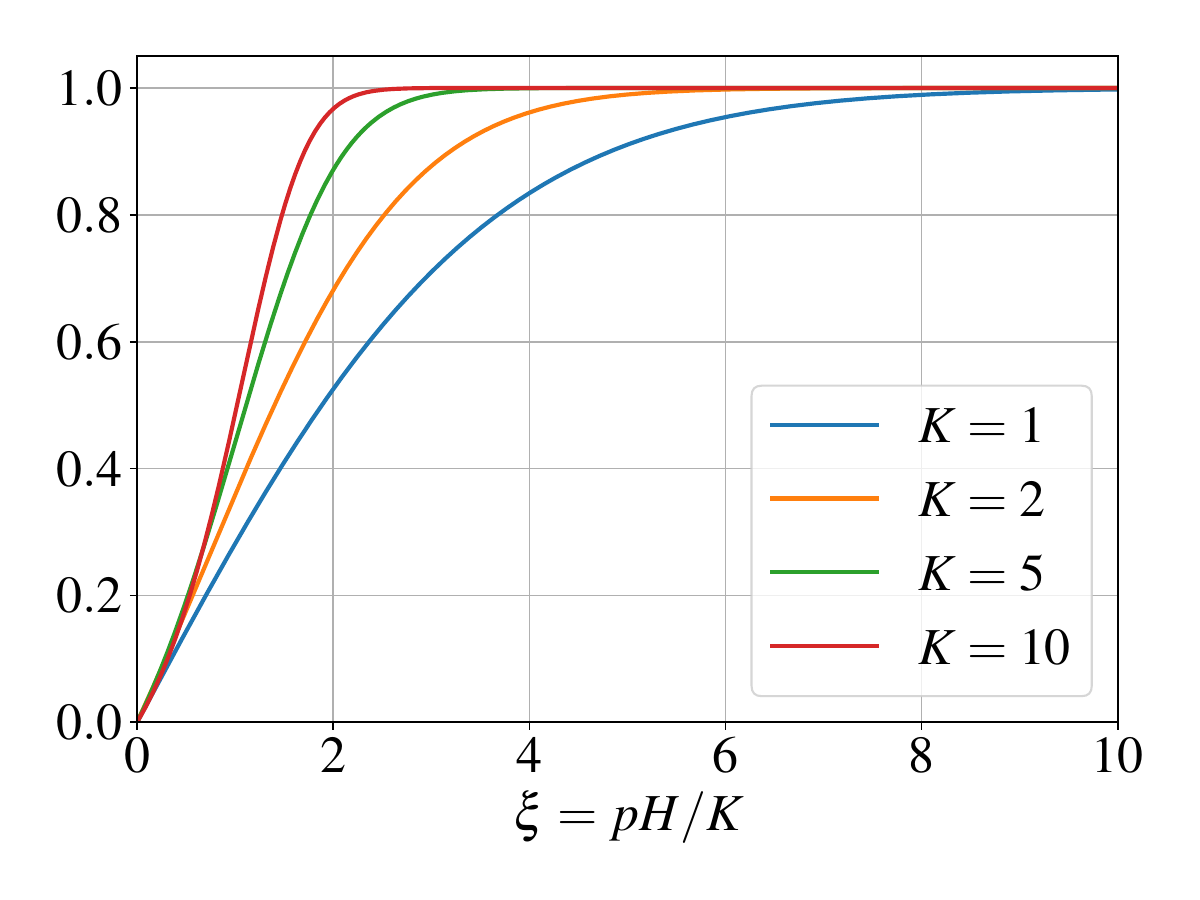}
    \caption{$C_\mathrm{SFT}(K, \xi)$ in \cref{eq:c_sft_gamma}.}
    \label{fig:sft_1}
  \end{subfigure}
  \begin{subfigure}[b]{0.48\textwidth}
    \centering
    \includegraphics[height=4.5cm]{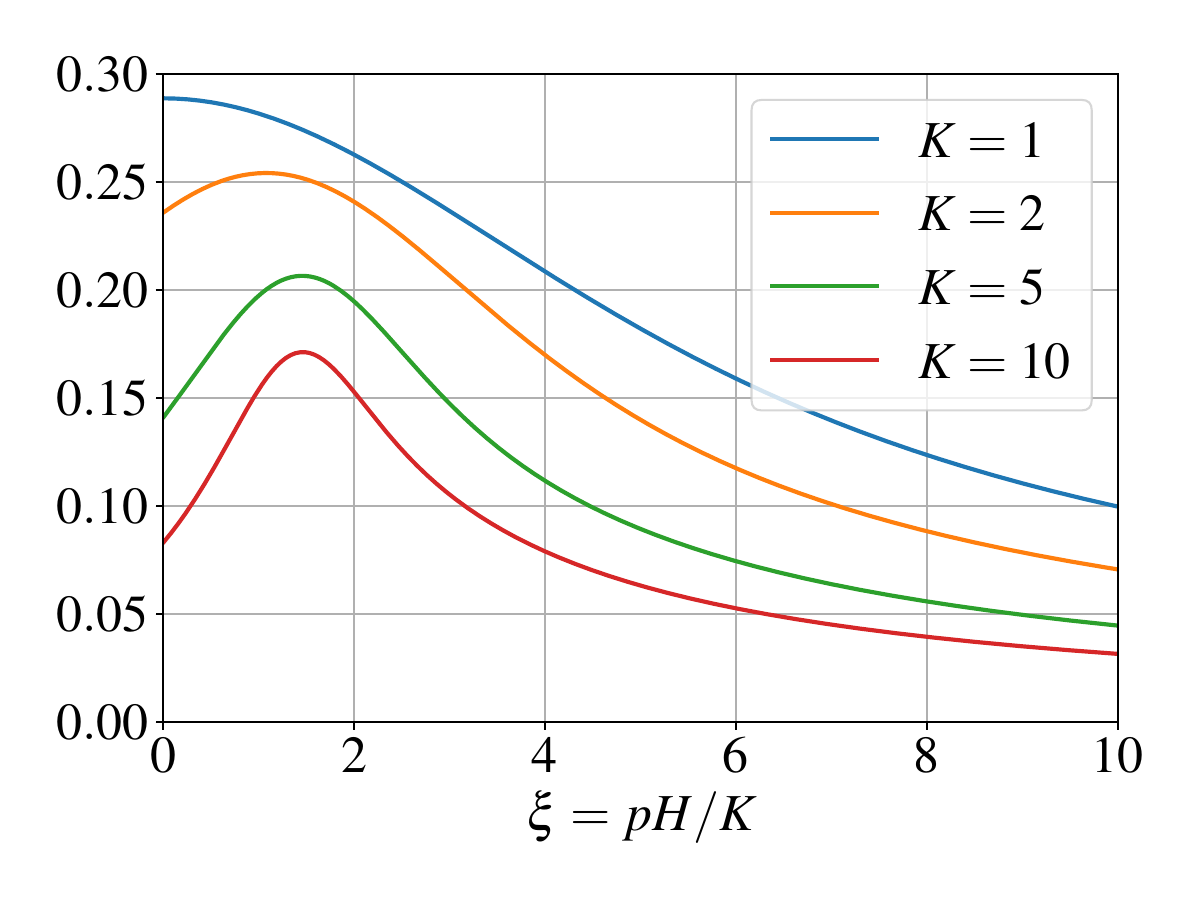}
    \caption{$\frac{C_\mathrm{SFT}(K, \xi)}{\sqrt{K}\xi}$ in \cref{eq:theorem_sft_1}.}
    \label{fig:sft_2}
  \end{subfigure}
  \label{fig:sft}
  \caption{Relations between $\xi$ and each coefficient regarding SFT.}
\end{figure}

\end{document}